\documentclass{article}

\usepackage[preprint]{neurips_2026}

\usepackage[utf8]{inputenc} 
\usepackage[T1]{fontenc}    
\usepackage{hyperref}       
\usepackage{url}            
\usepackage{booktabs}       
\usepackage{amsfonts}       
\usepackage{nicefrac}       
\usepackage{microtype}      
\usepackage{xcolor}         

\usepackage{amsmath}  
\usepackage{amssymb}  
\usepackage{amsthm}
\usepackage{enumitem}
\newtheorem{theorem}{Theorem}
\newtheorem{proposition}{Proposition}

\usepackage{algorithmic}
\usepackage{graphicx}

\newcommand{\SSIRepoURL}{\url{https://github.com/AtheavenMa/Spectral-Spatial-Interpretation}}

\title{Spectral-Spatial Interpretation}

\author{%
  Haotian Ma\textsuperscript{1}\qquad
  Ruqi Yang\textsuperscript{1,2,*}\qquad
  Philip Townsend\textsuperscript{2} \\
  \textsuperscript{1}Department of Computer Sciences \\
  \textsuperscript{2}Department of Forest and Wildlife Ecology \\
  University of Wisconsin--Madison \\
  \textsuperscript{*}Corresponding author: \texttt{ryang275@wisc.edu} \\
}

\begin{document}

\maketitle

\begin{abstract}
Post-hoc image explanations usually show \emph{where} a model looks using
spatial heatmaps or indicate \emph{which channels} matter using global channel
scores. These separate views are limiting for multispectral data because channel
relevance is often spatially localized and physically meaningful. We introduce
\textbf{Spectral-Spatial Interpretation (SSI)}, a full-grid attribution method
over position-channel players \((g,c)\). For each spatial position, SSI fixes
the observed surrounding context, defines a conditional game over the channels
at that position, and applies the Shapley value to this local game. The result
is a full spectral-spatial attribution matrix rather than a single heatmap or a
global channel ranking. SSI has a clear local interpretation. It satisfies Shapley
fairness within each position, and aggregating its channel attributions exactly
recovers the position-level contextual contrast obtained by replacing that
position with the baseline. SSI also differs from context-averaged Owen-style
hierarchical accounting, which uses the same position-channel grouping but
averages over external spatial coalitions. On HLS Burn Scar, SSI outperforms an
MC-Owen estimator across retained attribution metrics while using only about
$6\%$ of its evaluation budget. Across CUB-200, EuroSAT, and HLS Burn Scar
under both per-channel mean and zero baselines, SSI achieves the best joint
deletion and insertion AUC over dense, flat, and rank-1 baselines, yields
the strongest spectral hit rates and channel deletion, and leads spatial
localization on both annotated datasets.
Code is available at \SSIRepoURL.
\end{abstract}

\section{Introduction}
\label{sec:introduction}
Modern vision models increasingly operate on inputs with multiple physically
meaningful channels, from RGB fine-grained recognition to Sentinel-2 land-cover
classification and burn-scar mapping. Existing post-hoc explanations usually
answer \emph{where} the model looks, but not \emph{which spectral channels}
support the prediction at each location. For these data, high attribution to near-infrared at a vegetation boundary or
to short-wave infrared within a burn scar reflects different physical evidence
from attribution to visible bands. An explanation should therefore identify
where the evidence lies, which channels support it, and how much each channel
contributes at each location.

Existing methods collapse this structure in one of two ways. Spatial
attribution methods such as Grad-CAM~\citep{selvaraju2017grad},
RISE~\citep{petsiuk2018rise}, and Extremal
Perturbation~\citep{fong2019understanding} aggregate over channels and produce
a single heatmap. Spectral or channel attribution methods rank bands globally
using mutual information, attention weights, or SHAP-style
scores~\citep{lundberg2017unified}, averaging channel effects across space.
Even a two-stage pipeline that multiplies a spatial heatmap by a global channel
vector assumes a rank-1 structure, \(A(g,c)=s(g)q(c)\), meaning that the same
channel profile applies everywhere. This assumption fails when different regions have different spectral
signatures, as with red and blue plumage on a bird or vegetation and water in a
satellite scene (Figure~\ref{fig:fig1}). Channel importance depends on location.
The same band can be decisive in one region and irrelevant in another, so the
attribution target should be a localized position-channel score rather than a
global channel ranking.

We introduce \textbf{Spectral-Spatial Interpretation (SSI)}, a full-grid attribution method that assigns
a score $\Phi(g,c)$ to each grid cell $g$ and channel $c$ (Figure~\ref{fig:fig2}). For each position, SSI fixes the rest of the image
to its observed values and applies the standard Shapley value only to the channels at that position.
The contribution lies in the observed-context position game rather than a new allocation rule. The output
can be used directly as a joint attribution matrix or aggregated into spatial scores
$\sum_c\Phi(g,c)$ identifying \emph{where} evidence lies and spectral scores $\sum_g\Phi(g,c)$
identifying \emph{which channels} contribute, without assuming a rank-1 spatial-spectral factorization.
This observed-context formulation gives SSI local Shapley fairness and contextual aggregation
consistency. In particular, $\sum_c \Phi(g,c)$ equals the full-context contrast of replacing the entire position $g$
with the baseline. Unlike context-averaged Owen-style hierarchical attribution~\citep{owen1977values},
which averages over external spatial coalitions, SSI targets localized spectral evidence and
solves small channel-player games at each position.

\begin{figure}[t]
    \centering
    \includegraphics[width=\linewidth]{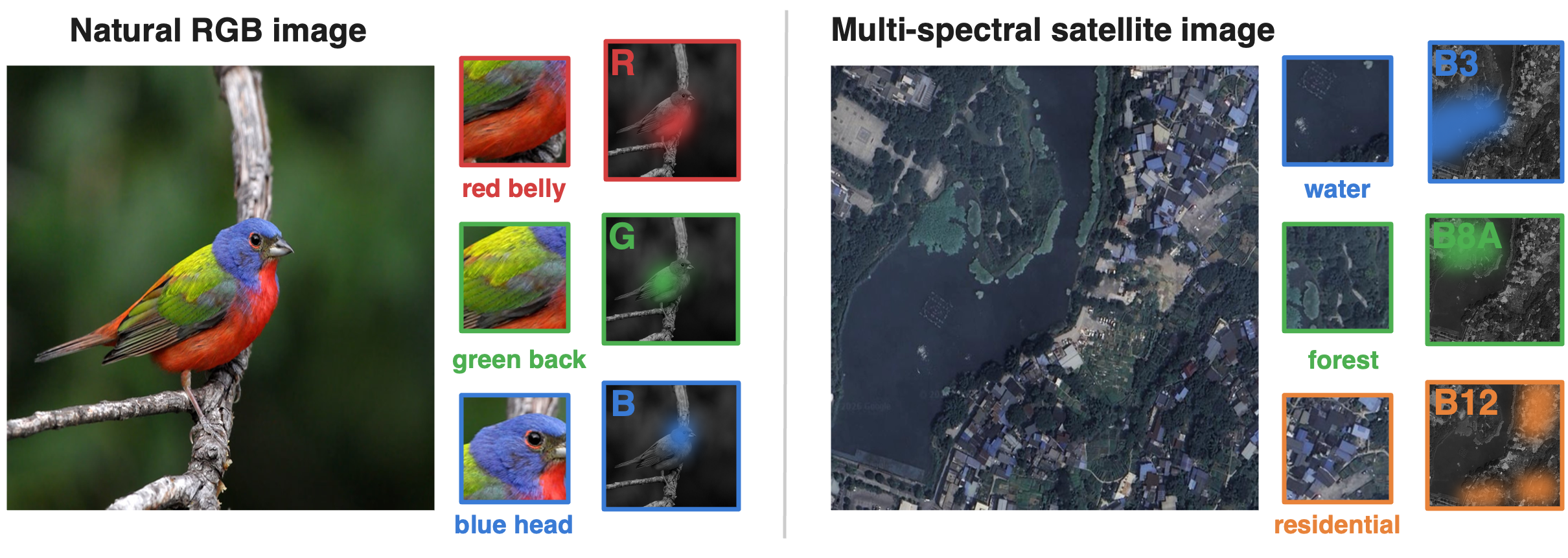}
    \caption{\textbf{Spectral channels carry spatially distinct physical information.} On a Painted Bunting (a), the red belly, green back, and blue head each align with a different RGB channel; on a Sentinel-2 scene (b), water, forest, and built-up areas align with B3, B8A, and B12. SSI is designed to recover this joint spectral-spatial structure.}
    \label{fig:fig1}
\end{figure}

We evaluate SSI on CUB-200~\citep{wah2011caltech},
EuroSAT~\citep{helber2019eurosat}, and HLS Burn
Scar~\citep{phillips2023hlsburnscar}. Across all three datasets, SSI
outperforms dense, flat, and rank-1 attribution baselines on joint
position-channel faithfulness metrics, and on HLS Burn Scar it further
outperforms an MC-Owen estimator while using only about \(6\%\) of its
evaluation budget.

Our contributions are as follows:
\begin{itemize}
\item We formulate spectral-spatial interpretation as attribution over position-channel players
$(g,c)$, targeting localized channel evidence rather than only a heatmap or a global channel ranking.
\item We define SSI as an observed-context attribution method that fixes the surrounding image context and applies standard Shapley allocation among channels at each position. We characterize its local Shapley fairness, contextual aggregation consistency, and distinction from context-averaged
Owen-style accounting.
\item Across CUB-200, EuroSAT, and HLS Burn Scar under both per-channel mean and zero baselines,
SSI achieves the best joint deletion, rank-1 deletion, and grid-sum insertion AUC over dense, flat,
and rank-1 baselines. It leads spectral hit rates and channel deletion on EuroSAT and HLS Burn Scar,
as well as spatial localization on both annotated datasets. On HLS Burn Scar, SSI also outperforms
MC-Owen using about 6\% of its evaluation budget.
\end{itemize}

\section{Related Work}

\paragraph{Post-hoc spatial attribution.}
A large family of explanation methods produces spatial heatmaps indicating which image regions drive a prediction. Gradient-based approaches such as Integrated Gradients~\citep{sundararajan2017axiomatic} and Grad-CAM~\citep{selvaraju2017grad} propagate target-class signals to pixels or activation maps. Perturbation-based approaches such as RISE~\citep{petsiuk2018rise} and Extremal Perturbation~\citep{fong2019understanding} measure prediction changes under spatial masks.

These methods are effective for identifying where evidence lies, but their explanations either aggregate over channels or treat input elements independently. They therefore do not provide a principled decomposition of which spectral channels matter at each location. 

\paragraph{Spectral and channel attribution.}
A complementary line of work ranks spectral bands or channels using mutual information~\citep{feng2014trivariate}, game-theoretic criteria~\citep{bruce2013game}, attention weights~\citep{hu2018squeeze, woo2018cbam}, or SHAP-style attribution~\citep{lundberg2017unified}. Such methods are useful for identifying globally important wavelengths, but they typically average channel relevance over space. This averaging can hide localized spectral evidence because a band may be unimportant on average yet decisive in a small region. SSI instead assigns attribution to each position-channel pair and preserves spatial variation in channel importance.

\paragraph{Joint spectral-spatial modeling.}
Architectures such as 3D CNNs~\citep{chen2016deep}, spectral-spatial
transformers~\citep{dang2022spectral}, and factorized architecture
search~\citep{zhong2022spectral} jointly exploit spectral and spatial
information but focus on prediction performance, and their internal
representations remain opaque.
Propagation-based explanations have been adapted for hyperspectral
classifiers~\citep{nagasubramanian2018explaining}, but remain architecture-specific.
SSI instead provides model-agnostic, game-theoretic attribution for
already-trained models.

\paragraph{Game-theoretic and hierarchical attribution.}
Shapley values~\citep{shapley1953value} and SHAP-style methods~\citep{lundberg2017unified} provide additive feature attributions but usually treat all interpretable units as flat players. This flat treatment does not reflect the natural position-channel hierarchy of multispectral images. The Shapley interaction index~\citep{grabisch1999axiomatic} and its higher-order extensions~\citep{sundararajan2020shapley, tsai2023faith} quantify synergies but do not exploit this two-level structure. Group-level Shapley methods~\citep{jullum2021groupshapley} treat each
predefined group as an atomic player and do not decompose contributions
within groups; methods addressing feature
dependence~\citep{aas2021explaining} estimate conditional expectations
rather than fixing observed context. SSI differs from both lines of work in its attribution units
and treatment of context.

\paragraph{Owen value.}
The Owen value~\citep{owen1977values}, building on games with coalition structures~\citep{aumann1974cooperative}, respects a priori coalitions by allocating value across and within groups. SSI uses the same position-channel grouping but targets observed-context conditional attribution rather than context-averaged hierarchical accounting.

\begin{figure}[t]
    \centering
    \includegraphics[width=\linewidth]{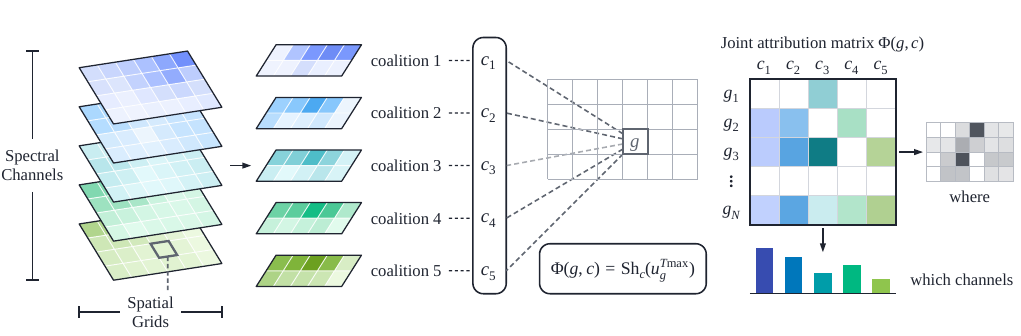}
    \caption{
    \textbf{Overview of SSI.}
    For each spatial position \(g\), SSI fixes the observed surrounding context and
    allocates the local contribution among channels, producing a full joint
    attribution matrix \(\Phi(g,c)\) for spatial and spectral interpretation.
    }
    \label{fig:fig2}
\end{figure}

\section{Method}
\label{sec:method}
SSI explains predictions using joint position-channel players $(g,c)$. It applies the standard Shapley value within each position, with the conditional position game providing the formulation for spectral-spatial explanation. We define the player space and the conditional game at each position, derive SSI as its Shapley value, and establish local fairness and contextual aggregation consistency. We also explain how SSI differs from globally efficient context-averaged accounting. Proofs and sampling details are in Appendix~\ref{app:method-details}.

\subsection{Spectral-Spatial Players}
\label{sec:formulation}

Let \(f\) be a trained model and let \(x\in\mathbb{R}^{C\times H\times W}\) be
an input with \(C\) channels. We partition the spatial domain into grid cells
indexed by \(\mathcal{G}\) and define the player set
\[
\mathcal{P}=\mathcal{G}\times\mathcal{C},
\qquad
(g,c)\in\mathcal{P}.
\]
For any subset \(S\subseteq\mathcal{P}\), let \(x_S\) denote the input with
players outside \(S\) replaced by a masking baseline. Let \(z_y(x_S)\) be the
target-class logit. We use the centered value function
\begin{equation}
v(S)=z_y(x_S)-z_y(x_\emptyset),
\qquad v(\emptyset)=0.
\label{eq:value}
\end{equation}
Adding a player may increase or decrease the target logit, so we do not
assume monotonicity.

The player set decomposes by spatial position:
\begin{equation}
\mathcal{P}=\bigsqcup_{g\in\mathcal{G}}S_g,
\qquad
S_g=\{(g,c):c\in\mathcal{C}\}.
\label{eq:decomp}
\end{equation}
The group \(S_g\) contains all channels at position \(g\). Channels within
\(S_g\) observe the same physical region through different spectral
measurements, whereas channels at different positions correspond to different
spatial regions.

For reference, we define two full-context contrasts. The player-level contrast is
\begin{equation}
\delta(g,c)
=
v(\mathcal{P})
-
v(\mathcal{P}\setminus\{(g,c)\}),
\label{eq:player-deletion}
\end{equation}
and the position-level contrast is
\begin{equation}
\Delta(g)
=
v(\mathcal{P})
-
v(\mathcal{P}\setminus S_g).
\label{eq:position-deletion}
\end{equation}
These compare the observed input to a baseline-replaced version while all
non-target players remain fixed.

\subsection{Conditional Position Games}
\label{sec:conditional-game}

SSI explains each position relative to the rest of the observed image. For a
target position \(g\), let
\[
B_g=\mathcal{P}\setminus S_g
\]
be all players outside that position. We define the conditional position game
\begin{equation}
w_g(R)
=
v(B_g\cup R)-v(B_g),
\qquad
R\subseteq S_g.
\label{eq:conditional-game}
\end{equation}
This game varies only the channels at position \(g\), while holding all other
positions fixed in their observed state. It satisfies
\[
w_g(\emptyset)=0,
\qquad
w_g(S_g)=v(\mathcal{P})-v(\mathcal{P}\setminus S_g)=\Delta(g).
\]
Thus, \(w_g\) is the local channel-allocation problem induced by the observed
surrounding context.

\subsection{Spectral-Spatial Interpretation}
\label{sec:ssi-value}

SSI assigns attribution to channel \(c\) at position \(g\) by applying the
Shapley value to the conditional game \(w_g\):
\begin{equation}
\Phi(g,c)
=
\mathrm{Sh}_c(w_g).
\label{eq:full-context}
\end{equation}
Since $|S_g|=C$, this is a standard Shapley value over $C$ players:
\begin{equation}
\Phi(g,c)
=
\sum_{R\subseteq S_g\setminus\{(g,c)\}}
\frac{|R|!(C-|R|-1)!}{C!}
\left[
w_g(R\cup\{(g,c)\})-w_g(R)
\right].
\label{eq:ssi-expanded}
\end{equation}
The output is a full attribution matrix
\[
\Phi\in\mathbb{R}^{C\times H\times W},
\]
defined for every grid cell and every channel.

SSI is one endpoint of a broader context-weighted family
$\alpha_\mu(g,c)=\sum_T \mu(T)\,\mathrm{Sh}_c(w_g^T)$
(Appendix~\ref{app:owen-relation}); SSI sets
$\mu(T)=\mathbf{1}\{T=\mathcal{G}\setminus\{g\}\}$, while the Owen value
averages over $T$ with Shapley weights. SSI does not average a channel's
contribution over all possible external spatial coalitions. Instead, it asks how
the channels at a position contribute in the surrounding image that was actually
observed.

\subsection{Properties}
\label{sec:ssi-properties}

The first property is a local characterization. For each fixed position \(g\),
SSI is the standard Shapley allocation of the conditional game \(w_g\).

\begin{theorem}[Local Shapley characterization]
\label{thm:local-ssi}
For each position \(g\), \(\Phi(g,\cdot)\) is the unique attribution over
channels in \(S_g\) satisfying local efficiency, within-position symmetry,
within-position null-player, and additivity for the conditional game \(w_g\).
\end{theorem}

Theorem~\ref{thm:local-ssi} applies the standard Shapley characterization to
the local game \(w_g\). SSI therefore allocates attribution within each spatial
position after conditioning on the observed context.

The second property links the channel-level scores to a spatial summary.

\begin{theorem}[Contextual aggregation consistency]
\label{thm:positive}
For every position \(g\),
\begin{equation}
\sum_{c\in\mathcal{C}}\Phi(g,c)
=
\Delta(g).
\label{eq:position-consistency}
\end{equation}
\end{theorem}

Thus, the channel-aggregated SSI map exactly equals the full-context
contrast of replacing the entire position by the baseline. This identity follows from local Shapley efficiency and requires no
monotonicity, supermodularity, or architectural assumption.

This local aggregation property is different from global value accounting. In
interacting games, the full-context contrasts of different positions need not
sum to the total value.

\begin{proposition}[Contextual aggregation and global efficiency can conflict]
\label{prop:global-efficiency-conflict}
For the class of all value functions \(v\), no attribution rule can guarantee
both
\[
\sum_{c\in\mathcal{C}}\alpha(g,c)=\Delta(g)
\quad\text{for every }g
\]
and
\[
\sum_{g\in\mathcal{G}}\sum_{c\in\mathcal{C}}\alpha(g,c)=v(\mathcal{P})
\]
for every game.
\end{proposition}

SSI provides position-wise conditional explanations with exact contextual
aggregation. It does not aim to allocate the entire prediction value globally
across all positions.

\subsection{Computation, Outputs, and Relation to Owen-Style Accounting}
\label{sec:practical}

In the main experiments, SSI evaluates \(\Phi(g,c)\) for every position
\(g\in\mathcal{G}\) and channel \(c\in\mathcal{C}\) on the full grid. No subset
selection is used, and \(\Phi(g,c)\) is reported for every position-channel pair.

We use permutation sampling over the channels within each
position \citep{castro2009polynomial}. In all experiments, SSI uses
\[
K=\min(2^C,10)
\]
permutations per grid cell, or enumerates all $2^C$ subsets when
$2^C\leq 10$. The sampling estimator is unbiased and concentrates as a
standard bounded Monte Carlo estimator; details are given in
Appendix~\ref{app:ssi-sampling}, and the corresponding total forward-pass
counts per image are reported in Appendix~\ref{app:computational_cost}.
For $C=3$ (CUB-200), this enumerates all channel subsets exactly; for
$C{=}13$ (EuroSAT) and $C{=}6$ (HLS), the empirical concentration bounds
remain below $10^{-2}$ (Appendix~\ref{app:empirical-concentration}).

The output of SSI can be evaluated directly as the joint attribution matrix $\Phi(g,c)$ or summarized as spatial scores $\sum_c \Phi(g,c)$ and spectral scores $\sum_g \Phi(g,c)$.

SSI uses the same position-channel grouping as Owen-style hierarchical values,
but estimates a different quantity. The Owen value provides globally efficient,
context-averaged accounting by averaging over external spatial coalitions before
allocating value within a position. SSI instead fixes the observed external
context and estimates the conditional game \(w_g\). In estimation, SSI samples
only a small \(C\)-player conditional game at each position, whereas MC-Owen
samples block-respecting permutations over the full position-channel set and
includes additional variance from external spatial contexts. Appendix~\ref{app:owen-relation} makes this relation
explicit, and \S\ref{sec:exp_mc_owen_hls} reports a focused empirical
comparison.

\section{Experiments}
\label{sec:experiments}

We evaluate the faithfulness of SSI's joint position-channel attribution $\Phi(g,c)$ using deletion and insertion AUC. We then examine how rank-1 spatial-spectral factorization affects this attribution and assess spatial and spectral aggregation through localization, channel deletion, and spectral hit rates. Qualitative case studies are presented in \S\ref{sec:exp_qualitative}, with full metric curves in the appendix. We also compare SSI with stronger gradient-based baselines on HLS Burn Scar and evaluate faithfulness using ROAD.

\subsection{Experimental Setup}
\label{sec:experimental-setup}

\paragraph{Datasets and models.}
We evaluate SSI on three datasets covering RGB images, multispectral land-cover classification, and burn-scar mapping. \textbf{CUB-200-2011}~\citep{wah2011caltech} contains RGB images of 200 bird species with bounding-box annotations. We fine-tune an ImageNet-pretrained ResNet-18 classifier. \textbf{EuroSAT}~\citep{helber2019eurosat} contains 13-band Sentinel-2 patches from 10 land-use/land-cover classes. We train a ResNet-18 classifier from scratch on all bands. \textbf{HLS Burn Scar}~\citep{phillips2023hlsburnscar} contains 6-band
Harmonized Landsat--Sentinel-2 patches with burn-scar masks. We train a
lightweight UNet for burn-scar segmentation and use the burn-class logit
averaged over ground-truth mask positions as the scalar response for
attribution. Training details and model performance are reported in Appendix~\ref{app:exp_details}.

\paragraph{Player space and baselines.}
All methods are evaluated on the same full position-channel player space $(g,c)$. To balance spatial resolution and computational cost, we use an
$8\times8$ grid on EuroSAT and HLS Burn Scar and a finer $16\times16$
grid on CUB-200 to capture more detailed spatial structure. We use a per-channel mean baseline as the primary masking value in all
datasets and methods; to verify robustness to baseline choice, we also
report results with a zero baseline. The semantics of the conditional
game $w_g$ depend on this choice, as is inherent to all
perturbation-based attribution. We compare SSI against three classes of attribution methods. The dense and flat baselines include Integrated Gradients (IG)~\citep{sundararajan2017axiomatic}, DeepSHAP~\citep{lundberg2017unified,shrikumar2017learning}, and KernelSHAP~\citep{lundberg2017unified}. Dense maps from IG and DeepSHAP are summed within each position-channel player; KernelSHAP directly treats each $(g,c)$ as an interpretable feature. We also include a rank-1 perturbation control based on RISE~\citep{petsiuk2018rise}. This method forms a factorized attribution map \(A(g,c)=s(g)q(c)\), where \(s(g)\) is a spatial importance score and \(q(c)\) is a global channel score obtained from the target-logit drop caused by masking channel \(c\) globally across all spatial positions. It tests whether a spatial heatmap multiplied by a global channel vector is sufficient for joint spectral-spatial explanation. Finally, SSI computes the full-context attribution $\Phi(g,c)$ for every grid cell and channel. On HLS Burn Scar, we also compare against Expected Gradients~\citep{erion2021expected}, Guided IG~\citep{kapishnikov2021guided}, and FullGrad~\citep{srinivas2019fullgradient}; these additional comparisons are reported separately in \S\ref{sec:exp_gradient_baselines}. Methods with dense channel-resolved attribution maps are aggregated spatially within grid cells while retaining the channel dimension for joint evaluation; channel summation is performed only for spatial summaries. Baseline descriptions are in Appendix~\ref{app:baseline-details}, and evaluation budgets for the original methods are reported in Appendix~\ref{app:computational_cost}.

\paragraph{Metrics.}
All evaluations use the target-class logit, matching the value function in
\S\ref{sec:formulation}. Lower deletion AUC indicates better attribution, whereas higher insertion AUC,
bbox mIoU, and spectral hit rates indicate better results. Joint deletion operates directly on ranked position-channel players. Rank-1 projection evaluates each method under the best rank-1 factorization of its attribution matrix into a spatial heatmap and a global channel profile. For insertion, we sum attribution scores over channels at each grid position and insert positions in decreasing order of the aggregated score, which evaluates the spatial summary and ensures comparability with rank-1 baselines like RISE. Spatial aggregation uses
$A_{\mathrm{sp}}(g)=\sum_c A(g,c)$, while spectral aggregation uses
$A_{\mathrm{ch}}(c)=\sum_g A(g,c)$.

For spatial localization, we threshold the channel-aggregated map at fixed
top-coverage levels and compute mIoU against available annotations. 

For spectral hit rates, we compare the top-ranked channels against
dataset-specific physical priors grounded in established spectral indices (e.g.,
Normalized Burn Ratio for burn scars, NDVI/NDRE for vegetation). Only bands
with well-documented physical correlations are included; classes without a
reliably positive spectral signature are excluded. The full prior specification
and scientific justification are given in
Appendix~\ref{app:spectral_priors}.

\subsection{Joint Faithfulness}
\label{sec:exp_joint}

SSI assigns attribution to full-grid position-channel players $(g,c)$. We therefore evaluate faithfulness by ranking the same players and measuring how efficiently the ranking removes target-class evidence. Our primary metric is joint deletion AUC: players are masked in decreasing attribution order, and the target-class logit is recorded after each step. Lower AUC indicates that the attribution ranking removes full-context evidence earlier.

We also report a rank-1 deletion diagnostic. For each method, we reshape the attribution scores into a matrix \(A\in\mathbb{R}^{C\times N}\), where \(N\) is the number of spatial cells, and retain only its leading rank-1 approximation. This factorized map has the form \(A(g,c)\approx a(g)b(c)\), corresponding to a spatial heatmap multiplied by a global channel profile. If such a factorization were sufficient, rank-1 deletion would preserve joint faithfulness.

Table~\ref{tab:joint_deletion} shows that SSI achieves the lowest joint
deletion AUC and the highest insertion AUC on all three datasets under
both the per-channel mean and zero baselines. On CUB-200 under the mean
baseline, SSI improves joint deletion over IG and DeepSHAP and achieves
the highest insertion AUC (.7132), surpassing RISE (.6886) despite
RISE using $8{,}000$ spatial masks versus $2048$ for SSI. The differences are larger on EuroSAT and HLS Burn Scar, where SSI removes
target evidence substantially earlier than all baselines under both deletion metrics. These results
support the use of $\Phi(g,c)$ as a full-context position-channel
attribution score and suggest that the advantage is not driven by a
particular baseline choice. SSI also performs consistently across the tested grid resolutions, as shown
in Appendix~\ref{app:grid_resolution}.

We also evaluate faithfulness using ROAD~\citep{rong2022road} on HLS Burn Scar. SSI achieves the highest LeRF AUC (.657) and lowest MoRF AUC (.1143) among nine evaluated methods under this complementary removal protocol. Appendix~\ref{app:road} reports the full comparison.

The rank-1 diagnostic further shows that factorized spectral-spatial
explanations are insufficient. RISE can provide strong spatial or
global-channel signals, but its rank-1 deletion AUC remains
substantially worse than SSI. The gap between joint and rank-1
deletion is smallest on CUB-200 ($C{=}3$) and largest on EuroSAT
($C{=}13$), suggesting that the rank-1 approximation degrades as
channel diversity increases. Thus, a spatial heatmap multiplied by a
global channel vector cannot replace the full joint attribution matrix.

\begin{table}[t]
\centering
\caption{\textbf{Joint faithfulness.} Deletion AUC (lower is better) and insertion AUC (higher is better). Values are mean $\pm$ bootstrap std. ``Joint'' and ``Rank-1'' evaluate position-channel deletion; ``Ins.'' evaluates grid-sum insertion. Best scores per column are in \textbf{bold}. Upper block: per-channel mean baseline; lower block: zero baseline.}
\label{tab:joint_deletion}
\small
\setlength{\tabcolsep}{2.8pt}
\resizebox{\linewidth}{!}{%
\begin{tabular}{@{}l ccc ccc ccc@{}}
\toprule
& \multicolumn{3}{c}{CUB-200}
& \multicolumn{3}{c}{EuroSAT}
& \multicolumn{3}{c}{HLS Burn Scar} \\
\cmidrule(lr){2-4} \cmidrule(lr){5-7} \cmidrule(lr){8-10}
Method
& Joint$\downarrow$ & Rk1$\downarrow$ & Ins.$\uparrow$
& Joint$\downarrow$ & Rk1$\downarrow$ & Ins.$\uparrow$
& Joint$\downarrow$ & Rk1$\downarrow$ & Ins.$\uparrow$ \\
\midrule
\multicolumn{10}{@{}l}{\textit{Per-channel mean baseline}} \\
\midrule
SSI
& \textbf{.0088{\scriptsize$\pm$.001}}
& \textbf{.0104{\scriptsize$\pm$.001}}
& \textbf{.7132{\scriptsize$\pm$.019}}
& \textbf{.1397{\scriptsize$\pm$.020}}
& \textbf{.1682{\scriptsize$\pm$.021}}
& \textbf{.7319{\scriptsize$\pm$.023}}
& \textbf{.1545{\scriptsize$\pm$.008}}
& \textbf{.1657{\scriptsize$\pm$.010}}
& \textbf{.6638{\scriptsize$\pm$.008}} \\
IG
& .0112{\scriptsize$\pm$.001}
& .0126{\scriptsize$\pm$.001}
& .2895{\scriptsize$\pm$.022}
& .2097{\scriptsize$\pm$.030}
& .2267{\scriptsize$\pm$.030}
& .6621{\scriptsize$\pm$.029}
& .2160{\scriptsize$\pm$.011}
& .2723{\scriptsize$\pm$.015}
& .6291{\scriptsize$\pm$.009} \\
DeepSHAP
& .0112{\scriptsize$\pm$.001}
& .0127{\scriptsize$\pm$.001}
& .2853{\scriptsize$\pm$.022}
& .2078{\scriptsize$\pm$.029}
& .2255{\scriptsize$\pm$.031}
& .6614{\scriptsize$\pm$.028}
& .2149{\scriptsize$\pm$.011}
& .2704{\scriptsize$\pm$.016}
& .6321{\scriptsize$\pm$.009} \\
KernelSHAP
& .0178{\scriptsize$\pm$.001}
& .0222{\scriptsize$\pm$.002}
& .3481{\scriptsize$\pm$.019}
& .2434{\scriptsize$\pm$.027}
& .2824{\scriptsize$\pm$.028}
& .5767{\scriptsize$\pm$.030}
& .2096{\scriptsize$\pm$.012}
& .2558{\scriptsize$\pm$.016}
& .5500{\scriptsize$\pm$.008} \\
RISE
& .0183{\scriptsize$\pm$.001}
& .0183{\scriptsize$\pm$.001}
& .6886{\scriptsize$\pm$.017}
& .2130{\scriptsize$\pm$.026}
& .1969{\scriptsize$\pm$.023}
& .7091{\scriptsize$\pm$.023}
& .2888{\scriptsize$\pm$.016}
& .2871{\scriptsize$\pm$.015}
& .6130{\scriptsize$\pm$.009} \\
\midrule
\multicolumn{10}{@{}l}{\textit{Zero baseline}} \\
\midrule
SSI
& \textbf{.0063{\scriptsize$\pm$.000}}
& \textbf{.0072{\scriptsize$\pm$.001}}
& \textbf{.5224{\scriptsize$\pm$.017}}
& \textbf{.0763{\scriptsize$\pm$.015}}
& \textbf{.0850{\scriptsize$\pm$.017}}
& \textbf{.6001{\scriptsize$\pm$.024}}
& \textbf{.1203{\scriptsize$\pm$.009}}
& \textbf{.1175{\scriptsize$\pm$.010}}
& \textbf{.6564{\scriptsize$\pm$.008}} \\
IG
& .0083{\scriptsize$\pm$.001}
& .0085{\scriptsize$\pm$.001}
& .0832{\scriptsize$\pm$.013}
& .1434{\scriptsize$\pm$.023}
& .1874{\scriptsize$\pm$.029}
& .4983{\scriptsize$\pm$.033}
& .1785{\scriptsize$\pm$.010}
& .1802{\scriptsize$\pm$.011}
& .6140{\scriptsize$\pm$.009} \\
DeepSHAP
& .0083{\scriptsize$\pm$.000}
& .0084{\scriptsize$\pm$.001}
& .0802{\scriptsize$\pm$.012}
& .1444{\scriptsize$\pm$.024}
& .1876{\scriptsize$\pm$.028}
& .4963{\scriptsize$\pm$.032}
& .1776{\scriptsize$\pm$.010}
& .1802{\scriptsize$\pm$.011}
& .6174{\scriptsize$\pm$.009} \\
KernelSHAP
& .0154{\scriptsize$\pm$.001}
& .0327{\scriptsize$\pm$.003}
& .1523{\scriptsize$\pm$.011}
& .1521{\scriptsize$\pm$.017}
& .1862{\scriptsize$\pm$.020}
& .4257{\scriptsize$\pm$.031}
& .1555{\scriptsize$\pm$.009}
& .1637{\scriptsize$\pm$.009}
& .5503{\scriptsize$\pm$.008} \\
RISE
& .0142{\scriptsize$\pm$.001}
& .0144{\scriptsize$\pm$.001}
& .4246{\scriptsize$\pm$.021}
& .1599{\scriptsize$\pm$.019}
& .1741{\scriptsize$\pm$.023}
& .5644{\scriptsize$\pm$.029}
& .1971{\scriptsize$\pm$.010}
& .2229{\scriptsize$\pm$.010}
& .6116{\scriptsize$\pm$.008} \\
\bottomrule
\end{tabular}
}

\end{table}

\subsection{Spatial and Spectral Aggregation}
\label{sec:exp_aggregation}

We assess whether the full position-channel attribution matrix remains
meaningful after aggregation. Spatial aggregation
$A_{\mathrm{sp}}(g)=\sum_c A(g,c)$ evaluates whether the resulting heatmap
localizes annotated regions; spectral aggregation
$A_{\mathrm{ch}}(c)=\sum_g A(g,c)$ evaluates whether the resulting channel
ranking identifies decision-relevant bands. For SSI, Theorem~\ref{thm:positive} gives this spatial aggregation a direct
contextual interpretation. The channel sum $\sum_c\Phi(g,c)$ equals the
full-context contrast of replacing the entire position by the baseline.

Figure~\ref{fig:aggregation_six_panel} summarizes all three aggregation
views. Panels~(a) and~(b) show spatial localization after channel
aggregation. On CUB-200, SSI achieves the strongest mIoU across all
coverage levels; on HLS Burn Scar the margin is larger, with SSI
reaching .389 mIoU at Top-10\% compared to .277 for DeepSHAP and .191
for RISE. KernelSHAP, the flat-Shapley baseline operating on the same
player granularity, lags substantially on both datasets. This is
consistent with Proposition~\ref{prop:flat-shapley}, which shows that
flat Shapley does not guarantee position-level aggregability, whereas
SSI's spatial sum equals the contextual contrast $\Delta(g)$ by
Theorem~\ref{thm:positive}.
Panels~(d) and~(e) show spectral hit rates on EuroSAT and HLS Burn
Scar. Across all reported $K$, SSI achieves the highest Hit@K on both
datasets, indicating that its aggregated channel ranking recovers
domain-derived spectral bands more reliably than the compared methods.
Panels~(c) and~(f) show channel deletion AUC on EuroSAT and HLS Burn
Scar. SSI achieves the lowest channel deletion AUC on both datasets,
confirming that its spectral summaries remove decision-relevant spectral
evidence earlier than the alternatives. The spatial and spectral summaries of the full attribution matrix therefore
remain informative in both domains.

\begin{figure*}[t]
    \centering
    \includegraphics[width=\textwidth]{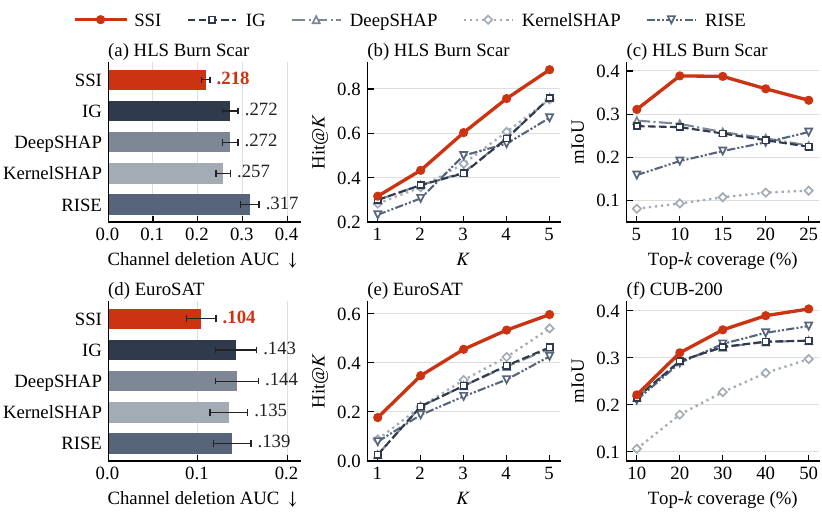}
    \caption{
    \textbf{Spatial and spectral aggregation views.}
    Panels~(c) and~(f) show spatial localization using mIoU at different
    coverage thresholds. Panels~(b) and~(e) report spectral Hit@K, while
    panels~(a) and~(d) show channel deletion AUC (lower is better).
    }
    \label{fig:aggregation_six_panel}
\end{figure*}

\subsection{Qualitative Case Studies}
\label{sec:exp_qualitative}

Figure~\ref{fig:fig4} visualizes representative localized channel evidence from channel-resolved methods. These cases illustrate attribution patterns rather than provide faithfulness evidence. Rank-1 controls such as RISE are omitted because they do not provide independent per-channel maps.

\textbf{Localized RGB evidence on CUB-200.}\quad
Figure~\ref{fig:fig4} shows a representative CUB-200 example. The
channel-aggregated SSI map concentrates on the bird region, while the
per-channel maps show different spatial supports across RGB channels. Regions with distinct plumage colors tend to receive relatively higher
attribution in their dominant color channel, reflecting the model's
learned spectral sensitivity of RGB bands.

\textbf{Localized physical bands on HLS Burn Scar.}\quad
SSI places high attribution near the burn-scar region and assigns substantial
contribution to physically meaningful bands such as Red and SWIR. Some baseline
channel maps appear more diffuse or highlight nearby structures less directly
tied to the annotated burn region. The spectral aggregation results also show that SSI more frequently
recovers burn-relevant bands.

\textbf{Class-specific spectral profiles on EuroSAT.}\quad
Vegetation classes are characterized by low red and high red-edge/NIR
reflectance~\citep{tucker1979red,clevers2002derivation}. SSI recovers this pattern. Red-edge and NIR bands receive the strongest
attribution while the red band ranks low. The compared baselines either miss the red-edge bands or
produce spectrally uniform scores. Figure~\ref{fig:aggregation_six_panel} shows a similar trend in the aggregate spectral results, where SSI achieves the lowest channel deletion AUC and leads at every Hit@$K$ threshold.

\begin{figure}[t]
    \centering
    \includegraphics[width=\linewidth]{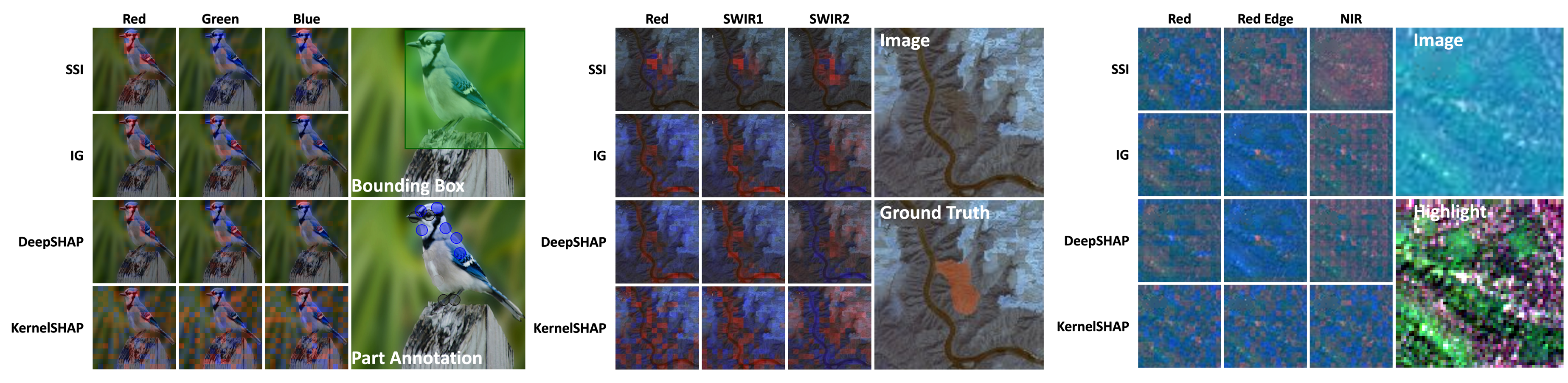}
    \caption{
    \textbf{Qualitative examples of spectral-spatial attribution.}
    Representative CUB-200, HLS Burn Scar, and EuroSAT examples showing
channel-aggregated maps and selected channel-level attributions. SSI illustrates
localized channel evidence that is not captured by a single heatmap or a global
channel ranking.
    }
    \label{fig:fig4}
\end{figure}

\subsection{Observed Context versus Context-Averaged Hierarchical Attribution}
\label{sec:exp_mc_owen_hls}
To compare observed-context attribution with context-averaged hierarchical accounting, we evaluate SSI against MC-Owen on HLS Burn Scar, the smallest player setting in our study ($|\mathcal{P}|=384$). Under the conservative forward-pass counting convention of Appendix~\ref{app:computational_cost}, MC-Owen with $M{=}200$ block-respecting permutations uses up to $M|\mathcal{P}|{=}76{,}800$ evaluations, whereas SSI uses up to $G\cdot K(C{+}1){=}4{,}480$ evaluations.
Table~\ref{tab:hls_owen_main}
shows that SSI outperforms MC-Owen on all retained metrics, including joint deletion,
rank-1 deletion, channel deletion, grid-sum insertion, spatial localization, and spectral
hit rates, while using roughly $6\%$ of the MC-Owen evaluation budget.
As shown in Appendix~\ref{app:empirical-concentration}, SSI's budget
yields uniform estimation errors below $10^{-2}$, suggesting that the current sampling budget is adequate
for the reported setting.
Both methods outperform the flat and rank-1 baselines in
Table~\ref{tab:joint_deletion}, confirming that the hierarchical grouping
is itself a useful inductive bias. The additional gain of SSI over MC-Owen reflects both the modeling choice
and the use of small $C$-player conditional games, which avoid the
external-coalition sampling variance inherent in MC-Owen.\footnotemark
\footnotetext{Permutation-based estimators converge at rate
$O(M^{-1/2})$~\citep{castro2009polynomial}. Increasing
$M$ to reduce noise would widen the cost gap substantially before
closing the performance gap.}
\begin{table}[t]
\centering
\caption{
\textbf{Observed-context SSI vs.\ context-averaged MC-Owen on HLS Burn Scar.}
Under the forward-pass counting convention of Appendix~\ref{app:computational_cost},
MC-Owen with $M{=}200$ block-respecting permutations uses up to $76{,}800$ evaluations,
and SSI uses up to $G\cdot K(C{+}1){=}4{,}480$ evaluations.
Grid-sum insertion ranks positions by channel-summed attribution.
Values are mean \(\pm\) bootstrap std.
}
\label{tab:hls_owen_main}
\small
\setlength{\tabcolsep}{3pt}
\resizebox{\linewidth}{!}{%
\begin{tabular}{@{}l ccc c c ccc ccc@{}}
\toprule
& \multicolumn{3}{c}{Deletion $\downarrow$}
& &
& \multicolumn{3}{c}{Spatial mIoU $\uparrow$}
& \multicolumn{3}{c}{Hit Rate $\uparrow$} \\
\cmidrule(lr){2-4} \cmidrule(lr){7-9} \cmidrule(lr){10-12}
Method
& Joint & Rk1 & Ch.
& Ins.$\uparrow$
&
& @5\% & @15\% & @25\%
& $K{=}1$ & $K{=}3$ & $K{=}5$ \\
\midrule
SSI
& \textbf{.155{\scriptsize$\pm$.008}}
& \textbf{.166{\scriptsize$\pm$.010}}
& \textbf{.218{\scriptsize$\pm$.009}}
& \textbf{.664{\scriptsize$\pm$.008}}
&
& \textbf{.311{\scriptsize$\pm$.012}}
& \textbf{.387{\scriptsize$\pm$.014}}
& \textbf{.332{\scriptsize$\pm$.017}}
& \textbf{.317{\scriptsize$\pm$.007}}
& \textbf{.603{\scriptsize$\pm$.016}}
& \textbf{.887{\scriptsize$\pm$.016}} \\
MC-Owen
& .173{\scriptsize$\pm$.010}
& .254{\scriptsize$\pm$.017}
& .261{\scriptsize$\pm$.017}
& .627{\scriptsize$\pm$.009}
&
& .299{\scriptsize$\pm$.018}
& .254{\scriptsize$\pm$.017}
& .224{\scriptsize$\pm$.017}
& .287{\scriptsize$\pm$.011}
& .460{\scriptsize$\pm$.017}
& .730{\scriptsize$\pm$.013} \\
\bottomrule
\end{tabular}
}

\end{table}

\subsection{Additional Gradient-Based Baselines}
\label{sec:exp_gradient_baselines}

On HLS Burn Scar, we compare SSI with Expected Gradients~\citep{erion2021expected},
Guided IG~\citep{kapishnikov2021guided}, and
FullGrad~\citep{srinivas2019fullgradient}. We use the
same position-channel evaluation metrics as in Table~\ref{tab:hls_owen_main}.
Table~\ref{tab:hls_gradient_main} shows that SSI performs best on all ten
reported metrics, covering joint and rank-1 deletion, channel deletion,
grid-sum insertion, spatial localization, and spectral hit rates. Additional method details appear in
Appendix~\ref{app:baseline-details}.

\begin{table}[h!]
\centering
\caption{\textbf{Additional gradient-based baselines on HLS Burn Scar.}
All methods are evaluated using the same position-channel player space.
Spatial mIoU is reported at top-5\%, top-15\%, and top-25\% coverage.
Values are mean $\pm$ reported variability. Best scores per column are in
\textbf{bold}.}
\label{tab:hls_gradient_main}
\small
\setlength{\tabcolsep}{3pt}
\resizebox{\linewidth}{!}{%
\begin{tabular}{@{}l ccc c ccc ccc@{}}
\toprule
 & \multicolumn{3}{c}{Deletion $\downarrow$}
 & & \multicolumn{3}{c}{Spatial mIoU $\uparrow$}
 & \multicolumn{3}{c}{Hit Rate $\uparrow$} \\
\cmidrule(lr){2-4} \cmidrule(lr){6-8} \cmidrule(lr){9-11}
Method & Joint & Rk1 & Ch. & Ins.$\uparrow$
 & @5\% & @15\% & @25\% & $K{=}1$ & $K{=}3$ & $K{=}5$ \\
\midrule
SSI
 & \textbf{.155{\scriptsize$\pm$.008}}
 & \textbf{.166{\scriptsize$\pm$.010}}
 & \textbf{.218{\scriptsize$\pm$.009}}
 & \textbf{.664{\scriptsize$\pm$.008}}
 & \textbf{.311{\scriptsize$\pm$.012}}
 & \textbf{.387{\scriptsize$\pm$.014}}
 & \textbf{.332{\scriptsize$\pm$.017}}
 & \textbf{.317{\scriptsize$\pm$.007}}
 & \textbf{.603{\scriptsize$\pm$.016}}
 & \textbf{.887{\scriptsize$\pm$.016}} \\
Expected Gradients
 & .185{\scriptsize$\pm$.010}
 & .238{\scriptsize$\pm$.015}
 & .275{\scriptsize$\pm$.017}
 & .637{\scriptsize$\pm$.010}
 & .292{\scriptsize$\pm$.016}
 & .284{\scriptsize$\pm$.016}
 & .236{\scriptsize$\pm$.018}
 & .287{\scriptsize$\pm$.005}
 & .510{\scriptsize$\pm$.020}
 & .720{\scriptsize$\pm$.012} \\
Guided IG
 & .185{\scriptsize$\pm$.011}
 & .244{\scriptsize$\pm$.016}
 & .280{\scriptsize$\pm$.017}
 & .630{\scriptsize$\pm$.011}
 & .296{\scriptsize$\pm$.016}
 & .283{\scriptsize$\pm$.015}
 & .233{\scriptsize$\pm$.018}
 & .293{\scriptsize$\pm$.005}
 & .500{\scriptsize$\pm$.019}
 & .703{\scriptsize$\pm$.010} \\
FullGrad
 & .278{\scriptsize$\pm$.014}
 & .277{\scriptsize$\pm$.014}
 & .581{\scriptsize$\pm$.013}
 & .543{\scriptsize$\pm$.015}
 & .131{\scriptsize$\pm$.018}
 & .135{\scriptsize$\pm$.015}
 & .128{\scriptsize$\pm$.011}
 & .000{\scriptsize$\pm$.000}
 & .333{\scriptsize$\pm$.000}
 & .667{\scriptsize$\pm$.000} \\
\bottomrule
\end{tabular}%
}
\end{table}

\section{Conclusion}
\label{sec:discussion}

We introduced Spectral-Spatial Interpretation (SSI), which explains
predictions of multi-channel image models at the level of joint
position-channel players $(g,c)$.  SSI fixes the observed surrounding context and allocates each
position's contribution among its channels via the Shapley value, producing a
full attribution matrix whose channel aggregation exactly recovers the
position-level contextual contrast. Across CUB-200, EuroSAT, and HLS Burn
Scar, SSI improves joint faithfulness, spectral hit rates, and channel
deletion over dense, flat, and rank-1 baselines, and outperforms MC-Owen on
HLS Burn Scar with roughly \(6\%\) of its evaluation budget.

Although SSI's estimation error grows only as $\sqrt{\ln C}$ and its
per-image cost grows only linearly in $C$
(Appendix~\ref{app:scalability}), applying SSI to hyperspectral sensors
with hundreds of correlated bands may benefit from channel screening or
adaptive grouping to reduce spectral redundancy before local Shapley
allocation. Validating such extensions on operational hyperspectral
missions remains future work. SSI may also be useful for imaging spectroscopy data from sensors such as EnMAP, where spectral channels carry spatially localized physical information. The observed-context formulation may extend to other multi-channel imaging domains with localized physical meaning, including biomedical microscopy and industrial materials inspection.

\paragraph{Code availability.}
The open-source SSI implementation and its usage documentation are available at
\SSIRepoURL.

\begin{ack}
We thank the reviewers and area chairs for their constructive feedback.

\textbf{Funding.}
No third-party funding or support was received for this work.

\textbf{Competing interests.}
The authors declare no competing interests.
\end{ack}

\bibliographystyle{plainnat}
\bibliography{references}

@inproceedings{selvaraju2017grad,
  title={Grad-{CAM}: Visual explanations from deep networks via gradient-based localization},
  author={Selvaraju, Ramprasaath R and Cogswell, Michael and Das, Abhishek and Vedantam, Ramakrishna and Parikh, Devi and Batra, Dhruv},
  booktitle={Proceedings of the IEEE International Conference on Computer Vision (ICCV)},
  pages={618--626},
  year={2017}
}

@inproceedings{sundararajan2017axiomatic,
  title={Axiomatic attribution for deep networks},
  author={Sundararajan, Mukund and Taly, Ankur and Yan, Qiqi},
  booktitle={Proceedings of the 34th International Conference on Machine Learning (ICML)},
  pages={3319--3328},
  year={2017}
}

@inproceedings{petsiuk2018rise,
  title={{RISE}: Randomized input sampling for explanation of black-box models},
  author={Petsiuk, Vitali and Das, Abir and Saenko, Kate},
  booktitle={British Machine Vision Conference (BMVC)},
  year={2018}
}

@inproceedings{fong2019understanding,
  title={Understanding deep networks via extremal perturbations and smooth masks},
  author={Fong, Ruth and Patrick, Mandela and Vedaldi, Andrea},
  booktitle={Proceedings of the IEEE/CVF International Conference on Computer Vision (ICCV)},
  pages={2950--2958},
  year={2019}
}

@inproceedings{lundberg2017unified,
  title={A unified approach to interpreting model predictions},
  author={Lundberg, Scott M and Lee, Su-In},
  booktitle={Advances in Neural Information Processing Systems (NeurIPS)},
  volume={30},
  year={2017}
}

@inproceedings{hu2018squeeze,
  title={Squeeze-and-excitation networks},
  author={Hu, Jie and Shen, Li and Sun, Gang},
  booktitle={Proceedings of the IEEE Conference on Computer Vision and Pattern Recognition (CVPR)},
  pages={7132--7141},
  year={2018}
}

@inproceedings{woo2018cbam,
  title={{CBAM}: Convolutional block attention module},
  author={Woo, Sanghyun and Park, Jongchan and Lee, Joon-Young and Kweon, In So},
  booktitle={European Conference on Computer Vision (ECCV)},
  pages={3--19},
  year={2018},
  publisher={Springer}
}

@inproceedings{bruce2013game,
  title={Game theory applied to big data analytics in geosciences and remote sensing},
  author={Bruce, Lori Mann},
  booktitle={Proceedings of the IEEE International Geoscience and Remote Sensing Symposium (IGARSS)},
  year={2013}
}

@article{zhong2022spectral,
  title={Spectral--spatial transformer network for hyperspectral image classification: A factorized architecture search framework},
  author={Zhong, Zilong and Li, Ying and Ma, Lingfei and Li, Jonathan and Zheng, Wei-Shi},
  journal={IEEE Transactions on Geoscience and Remote Sensing},
  volume={60},
  pages={1--15},
  year={2022},
  publisher={IEEE}
}

@article{chen2016deep,
  title={Deep feature extraction and classification of hyperspectral images based on convolutional neural networks},
  author={Chen, Yushi and Jiang, Hanlu and Li, Chunyang and Jia, Xiuping and Ghamisi, Pedram},
  journal={IEEE Transactions on Geoscience and Remote Sensing},
  volume={54},
  number={10},
  pages={6232--6251},
  year={2016},
  publisher={IEEE}
}

@article{dang2022spectral,
  title={Spectral-spatial attention transformer with dense connection for hyperspectral image classification},
  author={Dang, Lanxue and Weng, Libo and Dong, Weichuan and Li, Shenshen and Hou, Yane},
  journal={Computational Intelligence and Neuroscience},
  volume={2022},
  pages={7071485},
  year={2022},
  publisher={Hindawi}
}

@incollection{shapley1953value,
  title={A value for $n$-person games},
  author={Shapley, Lloyd S},
  booktitle={Contributions to the Theory of Games II},
  editor={Kuhn, Harold W and Tucker, Albert W},
  pages={307--317},
  year={1953},
  publisher={Princeton University Press}
}

@article{grabisch1999axiomatic,
  title={An axiomatic approach to the concept of interaction among players in cooperative games},
  author={Grabisch, Michel and Roubens, Marc},
  journal={International Journal of Game Theory},
  volume={28},
  number={4},
  pages={547--565},
  year={1999},
  publisher={Springer}
}

@inproceedings{sundararajan2020shapley,
  title={The {S}hapley {T}aylor interaction index},
  author={Sundararajan, Mukund and Dhamdhere, Kedar and Agarwal, Ashish},
  booktitle={Proceedings of the 37th International Conference on Machine Learning (ICML)},
  pages={9259--9268},
  year={2020}
}

@article{tsai2023faith,
  title={Faith-{S}hap: The faithful {S}hapley interaction index},
  author={Tsai, Che-Ping and Yeh, Chih-Kuan and Ravikumar, Pradeep},
  journal={Journal of Machine Learning Research},
  volume={24},
  number={94},
  pages={1--42},
  year={2023}
}

@article{jullum2021groupshapley,
  title={group{S}hapley: Efficient prediction explanation with {S}hapley values for feature groups},
  author={Jullum, Martin and Redelmeier, Annabelle and Aas, Kjersti},
  journal={arXiv preprint arXiv:2106.12228},
  year={2021},
  url={https://arxiv.org/abs/2106.12228}
}

@article{aas2021explaining,
  title={Explaining individual predictions when features are dependent: More accurate approximations to {S}hapley values},
  author={Aas, Kjersti and Jullum, Martin and L{\o}land, Anders},
  journal={Artificial Intelligence},
  volume={298},
  pages={103502},
  year={2022},
  publisher={Elsevier}
}

@incollection{owen1977values,
  title={Values of games with a priori unions},
  author={Owen, Guillermo},
  booktitle={Mathematical Economics and Game Theory: Essays in Honor of Oskar Morgenstern},
  editor={Henn, Rudolf and Moeschlin, Otto},
  series={Lecture Notes in Economics and Mathematical Systems},
  volume={141},
  pages={76--88},
  year={1977},
  publisher={Springer}
}

@article{aumann1974cooperative,
  title={Cooperative games with coalition structures},
  author={Aumann, Robert J and Dr{\`e}ze, Jacques H},
  journal={International Journal of Game Theory},
  volume={3},
  number={4},
  pages={217--237},
  year={1974},
  publisher={Springer}
}

@article{nagasubramanian2018explaining,
  title     = {Explaining hyperspectral imaging based plant disease identification:
               {3D CNN} and saliency maps},
  author    = {Nagasubramanian, Koushik and Jones, Sarah and Singh, Asheesh K
               and Singh, Arti and Ganapathysubramanian, Baskar and Sarkar, Soumik},
  journal   = {arXiv preprint arXiv:1804.08831},
  year      = {2018}
}

@techreport{wah2011caltech,
  title     = {The {Caltech-UCSD} Birds-200-2011 Dataset},
  author    = {Wah, Catherine and Branson, Steve and Welinder, Peter and Perona, Pietro and Belongie, Serge},
  institution = {California Institute of Technology},
  number    = {CNS-TR-2011-001},
  year      = {2011}
}

@inproceedings{shrikumar2017learning,
  title     = {Learning Important Features Through Propagating Activation Differences},
  author    = {Shrikumar, Avanti and Greenside, Peyton and Kundaje, Anshul},
  booktitle = {Proceedings of the 34th International Conference on Machine Learning (ICML)},
  pages     = {3145--3153},
  year      = {2017}
}

@article{helber2019eurosat,
  title     = {{EuroSAT}: A Novel Dataset and Deep Learning Benchmark
               for Land Use and Land Cover Classification},
  author    = {Helber, Patrick and Bischke, Benjamin and Dengel, Andreas
               and Borth, Damian},
  journal   = {IEEE Journal of Selected Topics in Applied Earth
               Observations and Remote Sensing},
  volume    = {12},
  number    = {7},
  pages     = {2217--2226},
  year      = {2019},
  publisher = {IEEE}
}

@misc{phillips2023hlsburnscar,
  title     = {{HLS Foundation Burnscars Dataset}},
  author    = {Phillips, Christopher and Roy, Sujit and Ankur, Kumar
               and Ramachandran, Rahul},
  year      = {2023},
  month     = aug,
  doi       = {10.57967/hf/0956},
  howpublished = {\url{https://huggingface.co/datasets/ibm-nasa-geospatial/hls_burn_scars}},
  note      = {Dataset. DOI: \url{https://doi.org/10.57967/hf/0956}}
}

@article{feng2014trivariate,
  title={Hyperspectral Band Selection Based on Trivariate Mutual Information and Clonal Selection},
  author={Feng, Jie and Jiao, Licheng and Zhang, Xiangrong and Sun, Tao},
  journal={IEEE Transactions on Geoscience and Remote Sensing},
  volume={52},
  number={7},
  pages={4092--4105},
  year={2014},
  doi={10.1109/TGRS.2013.2279591},
  publisher={IEEE}
}

@article{claverie2018harmonized,
  title     = {The {Harmonized Landsat and Sentinel-2} surface reflectance data set},
  author    = {Claverie, Martin and Ju, Junchang and Masek, Jeffrey G and Dungan, Jennifer L
               and Vermote, Eric F and Roger, Jean-Claude and Skakun, Sergii and Justice, Chris},
  journal   = {Remote Sensing of Environment},
  volume    = {219},
  pages     = {145--161},
  year      = {2018},
  publisher = {Elsevier}
}

@techreport{key2006landscape,
  title       = {Landscape Assessment ({LA}): Sampling and Analysis Methods},
  author      = {Key, Carl H and Benson, Nathan C},
  institution = {USDA Forest Service, Rocky Mountain Research Station},
  number      = {RMRS-GTR-164-CD},
  year        = {2006}
}

@article{garcia1991mapping,
  title     = {Mapping burns and natural reforestation using thematic mapper data},
  author    = {Garc{\'\i}a, Mar{\'\i}a J L and Caselles, Vicente},
  journal   = {Geocarto International},
  volume    = {6},
  number    = {1},
  pages     = {31--37},
  year      = {1991},
  publisher = {Taylor \& Francis}
}

@article{trigg2001mirbi,
  title     = {An evaluation of different bi-spectral spaces for discriminating burned shrub-savannah},
  author    = {Trigg, Simon and Flasse, Stephane},
  journal   = {International Journal of Remote Sensing},
  volume    = {22},
  number    = {13},
  pages     = {2641--2647},
  year      = {2001},
  publisher = {Taylor \& Francis}
}

@article{clevers2002derivation,
  title     = {Derivation of the red edge index using the {MERIS} standard band setting},
  author    = {Clevers, Jan G P W and De Jong, Steven M and Epema, Gerrit F
               and Van Der Meer, Freek and Bakker, Wim H and Skidmore, Andrew K
               and Scholte, Klass H},
  journal   = {International Journal of Remote Sensing},
  volume    = {23},
  number    = {16},
  pages     = {3169--3184},
  year      = {2002},
  publisher = {Taylor \& Francis}
}

@article{delegido2011evaluation,
  title     = {Evaluation of {Sentinel-2} red-edge bands for empirical estimation of green {LAI} and chlorophyll content},
  author    = {Delegido, Jes{\'u}s and Verrelst, Jochem and Alonso, Luis and Moreno, Jos{\'e}},
  journal   = {Sensors},
  volume    = {11},
  number    = {7},
  pages     = {7063--7081},
  year      = {2011},
  publisher = {MDPI}
}

@article{tucker1979red,
  title     = {Red and photographic infrared linear combinations for monitoring vegetation},
  author    = {Tucker, Compton J},
  journal   = {Remote Sensing of Environment},
  volume    = {8},
  number    = {2},
  pages     = {127--150},
  year      = {1979},
  publisher = {Elsevier}
}

@inproceedings{rouse1974monitoring,
  title     = {Monitoring vegetation systems in the {Great Plains} with {ERTS}},
  author    = {Rouse, John W and Haas, Rudolf H and Schell, John A and Deering, Donald W},
  booktitle = {Third Earth Resources Technology Satellite-1 Symposium},
  volume    = {1},
  pages     = {309--317},
  year      = {1974},
  organization = {NASA}
}

@article{zha2003use,
  title     = {Use of normalized difference built-up index in automatically mapping urban areas from {TM} imagery},
  author    = {Zha, Yong and Gao, Jay and Ni, Shaoxiang},
  journal   = {International Journal of Remote Sensing},
  volume    = {24},
  number    = {3},
  pages     = {583--594},
  year      = {2003},
  publisher = {Taylor \& Francis}
}

@article{xu2008new,
  title     = {A new index for delineating built-up land features in satellite imagery},
  author    = {Xu, Hanqiu},
  journal   = {International Journal of Remote Sensing},
  volume    = {29},
  number    = {14},
  pages     = {4269--4276},
  year      = {2008},
  publisher = {Taylor \& Francis}
}

@article{castro2009polynomial,
  title     = {Polynomial calculation of the {Shapley} value based on sampling},
  author    = {Castro, Javier and G{\'o}mez, Daniel and Tejada, Juan},
  journal   = {Computers \& Operations Research},
  volume    = {36},
  number    = {5},
  pages     = {1726--1730},
  year      = {2009},
  publisher = {Elsevier}
}

@inproceedings{rong2022road,
  title = {A Consistent and Efficient Evaluation Strategy for Attribution Methods},
  author = {Rong, Yao and Leemann, Tobias and Borisov, Vadim and Kasneci, Gjergji and Kasneci, Enkelejda},
  booktitle = {Proceedings of the 39th International Conference on Machine Learning},
  volume = {162},
  series = {Proceedings of Machine Learning Research},
  pages = {18770--18795},
  year = {2022},
  publisher = {PMLR},
  url = {https://proceedings.mlr.press/v162/rong22a.html}
}

@article{erion2021expected,
  title = {Improving performance of deep learning models with axiomatic attribution priors and expected gradients},
  author = {Erion, Gabriel and Janizek, Joseph D. and Sturmfels, Pascal and Lundberg, Scott M. and Lee, Su-In},
  journal = {Nature Machine Intelligence},
  volume = {3},
  pages = {620--631},
  year = {2021},
  doi = {10.1038/s42256-021-00343-w}
}

@inproceedings{kapishnikov2021guided,
  title = {Guided Integrated Gradients: An Adaptive Path Method for Removing Noise},
  author = {Kapishnikov, Andrei and Venugopalan, Subhashini and Avci, Besim and Wedin, Ben and Terry, Michael and Bolukbasi, Tolga},
  booktitle = {Proceedings of the IEEE/CVF Conference on Computer Vision and Pattern Recognition},
  pages = {5050--5058},
  year = {2021},
  url = {https://openaccess.thecvf.com/content/CVPR2021/html/Kapishnikov_Guided_Integrated_Gradients_An_Adaptive_Path_Method_for_Removing_Noise_CVPR_2021_paper.html}
}

@inproceedings{srinivas2019fullgradient,
  title = {Full-Gradient Representation for Neural Network Visualization},
  author = {Srinivas, Suraj and Fleuret, Fran{\c{c}}ois},
  booktitle = {Advances in Neural Information Processing Systems},
  volume = {32},
  year = {2019},
  url = {https://proceedings.neurips.cc/paper_files/paper/2019/hash/80537a945c7aaa788ccfcdf1b99b5d8f-Abstract.html}
}

\clearpage
\appendix

\section{Additional Method Details}
\label{app:method-details}

This appendix provides proofs of the local Shapley characterization
(Appendix~\ref{app:local-axioms}), contextual aggregation consistency
(Appendix~\ref{app:proof-contextual-aggregation}), and the conflict between
contextual aggregation and global efficiency
(Appendix~\ref{app:proof-global-conflict}). The sampling estimator and
concentration bound appear in Appendix~\ref{app:ssi-sampling}. We discuss
flat Shapley in Appendix~\ref{app:flat-family} and the relation to
context-averaged Owen-style accounting in Appendix~\ref{app:owen-relation}.

\subsection{Local Shapley Axioms}
\label{app:local-axioms}

Fix a spatial position \(g\) and its conditional game
\[
w_g:2^{S_g}\to\mathbb{R},
\qquad
w_g(R)=v(B_g\cup R)-v(B_g).
\]
A local channel attribution is a map \(\beta_g:S_g\to\mathbb{R}\). SSI uses the
standard Shapley value on this local game. The relevant axioms are local rather
than global.

\begin{description}
\item[(L1) Local efficiency.]
\[
\sum_{p\in S_g}\beta_g(p)=w_g(S_g)-w_g(\emptyset)=w_g(S_g).
\]

\item[(L2) Within-position symmetry.]
If \(p,q\in S_g\) satisfy
\[
w_g(R\cup\{p\})=w_g(R\cup\{q\})
\]
for every \(R\subseteq S_g\setminus\{p,q\}\), then
\[
\beta_g(p)=\beta_g(q).
\]

\item[(L3) Within-position null player.]
If \(p\in S_g\) satisfies
\[
w_g(R\cup\{p\})=w_g(R)
\]
for every \(R\subseteq S_g\setminus\{p\}\), then
\[
\beta_g(p)=0.
\]

\item[(L4) Additivity.]
For two conditional games \(w_g^{(1)}\) and \(w_g^{(2)}\) on the same channel
set,
\[
\beta_g^{w_g^{(1)}+w_g^{(2)}}(p)
=
\beta_g^{w_g^{(1)}}(p)+\beta_g^{w_g^{(2)}}(p).
\]
\end{description}

\begin{proof}[Proof of Theorem~\ref{thm:local-ssi}]
For a fixed position \(g\), \(w_g\) is an ordinary cooperative game over the
finite player set \(S_g\). The Shapley value is the unique attribution
satisfying efficiency, symmetry, null-player, and additivity on this game.
Applying the standard Shapley characterization to \(w_g\) gives
\[
\beta_g(p)=\mathrm{Sh}_p(w_g),
\qquad p\in S_g.
\]
Substituting \(p=(g,c)\) yields
\[
\Phi(g,c)=\mathrm{Sh}_c(w_g),
\]
which is Eq.~\eqref{eq:full-context}.
\end{proof}

\subsection{Proof of Contextual Aggregation Consistency}
\label{app:proof-contextual-aggregation}

\begin{proof}[Proof of Theorem~\ref{thm:positive}]
By local Shapley efficiency in the conditional game \(w_g\),
\[
\sum_{c\in\mathcal{C}}\Phi(g,c)
=
w_g(S_g)-w_g(\emptyset).
\]
Since \(w_g(\emptyset)=0\) and
\[
w_g(S_g)=v(B_g\cup S_g)-v(B_g),
\]
we have
\[
\sum_{c\in\mathcal{C}}\Phi(g,c)
=
v(B_g\cup S_g)-v(B_g).
\]
Using \(B_g=\mathcal{P}\setminus S_g\) and \(B_g\cup S_g=\mathcal{P}\),
\[
\sum_{c\in\mathcal{C}}\Phi(g,c)
=
v(\mathcal{P})-v(\mathcal{P}\setminus S_g)
=
\Delta(g).
\]
\end{proof}

\subsection{Contextual Aggregation Versus Global Efficiency}
\label{app:proof-global-conflict}

\begin{proof}[Proof of Proposition~\ref{prop:global-efficiency-conflict}]
Suppose an attribution rule satisfied contextual aggregation for every position:
\[
\sum_{c\in\mathcal{C}}\alpha(g,c)=\Delta(g).
\]
Summing over \(g\) gives
\[
\sum_{g\in\mathcal{G}}\sum_{c\in\mathcal{C}}\alpha(g,c)
=
\sum_{g\in\mathcal{G}}\Delta(g).
\]
If the same rule also satisfied global efficiency, then every game would require
\[
\sum_{g\in\mathcal{G}}\Delta(g)=v(\mathcal{P}).
\]
This identity is false in general.

As a counterexample, consider two spatial positions \(g,h\), each with one
channel, so \(\mathcal{P}=\{a,b\}\). Define
\[
v(\{a,b\})=1,
\qquad
v(\emptyset)=v(\{a\})=v(\{b\})=0.
\]
Then
\[
\Delta(g)=v(\{a,b\})-v(\{b\})=1,
\qquad
\Delta(h)=v(\{a,b\})-v(\{a\})=1,
\]
so \(\Delta(g)+\Delta(h)=2\), while \(v(\mathcal{P})=1\). Therefore contextual
aggregation and global efficiency cannot be guaranteed simultaneously for all
games.
\end{proof}

\subsection{Sampling Estimator and Concentration}
\label{app:ssi-sampling}

For a position \(g\), exact SSI evaluates all subsets of \(S_g\):
\[
\Phi(g,c)
=
\sum_{R\subseteq S_g\setminus\{(g,c)\}}
\frac{|R|!(C-|R|-1)!}{C!}
\left[
w_g(R\cup\{(g,c)\})-w_g(R)
\right].
\]
This requires \(2^C\) subset evaluations per spatial position.

For larger \(C\), we estimate the same Shapley value by permutation sampling.
Let \(\pi\) be a uniformly random permutation of the channels in \(S_g\), and
let \(\mathrm{Pred}_{\pi}(c)\) denote the channels preceding \(c\). Define
\[
M_{g,c}(\pi)
=
w_g(\mathrm{Pred}_{\pi}(c)\cup\{(g,c)\})
-
w_g(\mathrm{Pred}_{\pi}(c)).
\]
With \(K\) independent permutations,
\begin{equation}
\widehat{\Phi}_K(g,c)
=
\frac{1}{K}\sum_{k=1}^K M_{g,c}(\pi_k).
\label{eq:ssi-mc}
\end{equation}

\begin{proposition}[Unbiasedness and concentration of SSI sampling]
\label{prop:ssi_sampling_concentration}
For every \(g,c\),
\[
\mathbb{E}[\widehat{\Phi}_K(g,c)]=\Phi(g,c).
\]
If \(M_{g,c}(\pi)\in[a_{g,c},b_{g,c}]\) for all permutations and
\[
R=\max_{g,c}(b_{g,c}-a_{g,c}),
\]
then for any \(\varepsilon>0\),
\[
\Pr\!\left(
\max_{g,c}
|\widehat{\Phi}_K(g,c)-\Phi(g,c)|
\geq \varepsilon
\right)
\leq
2|\mathcal G|C
\exp\!\left(
-\frac{2K\varepsilon^2}{R^2}
\right).
\]
Equivalently, with probability at least \(1-\delta\),
\[
\max_{g,c}
|\widehat{\Phi}_K(g,c)-\Phi(g,c)|
\leq
R
\sqrt{
\frac{\log(2|\mathcal G|C/\delta)}{2K}
}.
\]
\end{proposition}

\begin{proof}
For fixed \(g,c\), the Shapley value is the expectation of the marginal
contribution over a uniformly random channel permutation:
\[
\Phi(g,c)=\mathbb{E}_{\pi}[M_{g,c}(\pi)].
\]
Thus \(\widehat{\Phi}_K(g,c)\) is unbiased. Since the samples are independent
and bounded in an interval of width at most \(R\), Hoeffding's inequality gives
the stated tail bound for each \(g,c\). A union bound over all
\(|\mathcal G|C\) position-channel players gives the result.
\end{proof}

The bound does not imply that a fixed small \(K\) is universally sufficient for
every possible game. It shows that SSI estimates \(G\) local \(C\)-player games,
rather than a single full \(GC\)-player game. In our experiments we use
\(K=\min(2^C,10)\) permutations per grid cell, or enumerate all \(2^C\) subsets
when \(2^C\leq 10\); the corresponding total forward-pass counts are reported
in Appendix~\ref{app:computational_cost}.

\subsection{Empirical Instantiation of the Concentration Bound}
\label{app:empirical-concentration}

On CUB-200 ($C{=}3$), $2^C{=}8\leq10$, so SSI enumerates all channel
subsets exactly and no sampling variance arises. For the two datasets
where permutation sampling is used, we measure the empirical
marginal-contribution range $R{=}\max_{g,c}(b_{g,c}-a_{g,c})$ and
substitute into Proposition~\ref{prop:ssi_sampling_concentration} at
confidence $1{-}\delta{=}0.95$. The uniform bound is
\[
R\sqrt{\frac{\ln(2|\mathcal{G}|C/\delta)}{2K}}.
\]

\paragraph{EuroSAT} ($|\mathcal{G}|{=}64$, $C{=}13$, $K{=}10$).
The observed range is $R{\approx}0.0093$. The logarithmic term is
$\ln(2\times64\times13\,/\,0.05)=\ln(33280)\approx10.41$, giving
\[
0.0093\times\sqrt{\frac{10.41}{20}}
\;=\;0.0093\times0.721
\;\approx\;0.0067.
\]
This is a uniform bound over all $|\mathcal{G}|C{=}832$
position-channel players simultaneously.

\paragraph{HLS Burn Scar} ($|\mathcal{G}|{=}64$, $C{=}6$, $K{=}10$).
The observed range is $R{\approx}0.0062$. The logarithmic term is
$\ln(2\times64\times6\,/\,0.05)=\ln(15360)\approx9.64$, giving
\[
0.0062\times\sqrt{\frac{9.64}{20}}
\;=\;0.0062\times0.694
\;\approx\;0.0043.
\]
This is a uniform bound over all $|\mathcal{G}|C{=}384$ players.

\paragraph{Summary.}
Both bounds are well below the inter-method gaps reported in
Tables~\ref{tab:joint_deletion} and~\ref{tab:hls_owen_main},
confirming that $K{=}10$ provides sufficient estimation precision
for the reported comparisons.

\subsection{Relation to Flat Shapley}
\label{app:flat-family}

Flat Shapley treats every position-channel pair \((g,c)\) as an independent
player in a single game over \(\mathcal{P}\). SSI differs because it fixes the observed external context \(B_g\) and
allocates value only among the channels in \(S_g\). Flat Shapley averages over arbitrary subsets of \(\mathcal{P}\),
including mixed spatial configurations in which only some channels of another
position are present.

The following example shows that flat Shapley can change the summed attribution
assigned to a position even when the quotient game over positions is unchanged.

\begin{proposition}
\label{prop:flat-shapley}
The flat Shapley value \(\phi\) does not satisfy position-level hierarchical
aggregability. There exist value functions \(v_1,v_2\) with identical quotient
games over spatial positions but
\[
\sum_c \phi^{v_1}(g,c)
\neq
\sum_c \phi^{v_2}(g,c).
\]
\end{proposition}

\begin{proof}
Consider two spatial positions with
\[
S_g=\{a_1,a_2\},
\qquad
S_h=\{b\}.
\]
Let \(v_1\) be the unanimity game with
\[
v_1(\{a_1,a_2,b\})=1
\]
and all other subsets zero. Then
\[
\phi^{v_1}(a_1)=\phi^{v_1}(a_2)=\phi^{v_1}(b)=\frac{1}{3},
\]
so
\[
\sum_{a\in S_g}\phi^{v_1}(a)=\frac{2}{3}.
\]

Define \(v_2\) by
\[
v_2(S)=1
\quad\text{iff}\quad
b\in S
\text{ and }
S\cap\{a_1,a_2\}\neq\emptyset,
\]
and \(v_2(S)=0\) otherwise. The quotient games over positions are identical:
\[
v_i^Q(\emptyset)=0,\quad
v_i^Q(\{g\})=0,\quad
v_i^Q(\{h\})=0,\quad
v_i^Q(\{g,h\})=1
\]
for \(i=1,2\). However,
\[
\phi^{v_2}(a_1)=\phi^{v_2}(a_2)=\frac{1}{6},
\]
so
\[
\sum_{a\in S_g}\phi^{v_2}(a)=\frac{1}{3}.
\]
Thus the position-level summed attribution changes even though the quotient
game over positions is the same.
\end{proof}

\subsection{Relation to Context-Averaged Hierarchical Accounting}
\label{app:owen-relation}

SSI is not a context-averaged global accounting value. This section records its
relation to the classical hierarchical value obtained by averaging over
external spatial coalitions, and explains why the corresponding Monte Carlo
estimator has a different sampling burden.

Define the quotient game over spatial positions:
\begin{equation}
v^Q(T)
=
v\!\left(\bigcup_{g'\in T}S_{g'}\right),
\qquad
T\subseteq\mathcal{G}.
\label{eq:app-quotient-game}
\end{equation}
For a target position \(g\) and an external spatial coalition
\(T\subseteq\mathcal{G}\setminus\{g\}\), define the context-conditioned internal
game
\begin{equation}
w_g^T(R)
=
v\!\left(
R\cup\bigcup_{g'\in T}S_{g'}
\right)
-
v\!\left(
\bigcup_{g'\in T}S_{g'}
\right),
\qquad
R\subseteq S_g.
\label{eq:app-internal-game}
\end{equation}
SSI corresponds to
\[
T=T_{\max}=\mathcal{G}\setminus\{g\},
\qquad
w_g^{T_{\max}}(R)=w_g(R).
\]

More generally, define a context-weighted hierarchical family
\begin{equation}
\alpha_\mu(g,c)
=
\sum_{T\subseteq\mathcal{G}\setminus\{g\}}
\mu(T)\,\mathrm{Sh}_c(w_g^T),
\qquad
\mu(T)\geq0,\quad \sum_T\mu(T)=1.
\label{eq:app-parametric-family}
\end{equation}
SSI is the point-mass choice
\[
\mu_{\mathrm{SSI}}(T)=\mathbf{1}\{T=T_{\max}\}.
\]
The Owen value uses Shapley weights over external spatial coalitions:
\[
\mu_{\mathrm{Owen}}(T)
=
\frac{|T|!(|\mathcal{G}|-|T|-1)!}{|\mathcal{G}|!}.
\]
Thus, SSI and Owen-style attribution use the same position-channel grouping but
estimate different quantities. SSI conditions on the observed context, whereas
Owen averages over external contexts.

\begin{theorem}[Owen value for context-averaged hierarchical accounting]
\label{thm:app-owen-characterization}
Suppose an attribution over position-channel players satisfies global
efficiency, position-level aggregability depending only on the quotient game,
quotient-level Shapley consistency, within-position symmetry, within-position
null-player, and additivity. Then the unique attribution is
\begin{equation}
\psi(g,c)
=
\sum_{T\subseteq\mathcal{G}\setminus\{g\}}
\frac{|T|!(|\mathcal{G}|-|T|-1)!}{|\mathcal{G}|!}
\,
\mathrm{Sh}_c(w_g^T).
\label{eq:app-owen}
\end{equation}
\end{theorem}

\begin{proof}
The player set \(\mathcal{P}\) is partitioned into a priori unions
\(\{S_g\}_{g\in\mathcal{G}}\). The quotient game \(v^Q\) is the cooperative
game over spatial positions, and each \(w_g^T\) is the internal game over
channels at position \(g\) under external spatial context \(T\). Owen's
characterization for games with a priori unions states that the unique value
satisfying global efficiency, quotient-level Shapley consistency, and
within-union Shapley fairness is the Owen value. Applying this theorem to the
partition \(\{S_g\}_{g\in\mathcal{G}}\) gives Eq.~\eqref{eq:app-owen}.
\end{proof}

This theorem characterizes a context-averaged accounting objective. SSI does
not estimate \(\psi\) and does not use Eq.~\eqref{eq:app-owen} as its
computational target.

The sampling burden also differs. MC-Owen estimates Eq.~\eqref{eq:app-owen} by
sampling block-respecting permutations, which randomize both the channel order
within a position and the external spatial context preceding that position. Let
\(Z_{g,c}\) be one MC-Owen marginal sample for player \((g,c)\), and let \(T\)
be the sampled set of external spatial positions preceding \(g\). Then
\[
\mathrm{Var}(Z_{g,c})
=
\mathbb{E}_T[\mathrm{Var}(Z_{g,c}\mid T)]
+
\mathrm{Var}_T\!\left(\mathbb{E}[Z_{g,c}\mid T]\right).
\]
The first term is within-position channel-ordering variance. The second is
external-context variance caused by averaging over spatial coalitions. This
second term is absent from SSI because SSI fixes the observed external context.
When evidence is spatially localized, this external-context variance can make
MC-Owen estimates noisy unless many block-respecting permutations are sampled.

\subsection{Optional Region-Restricted SSI}
\label{app:roi}

The main experiments compute SSI for every spatial position. For applications
requiring local explanations or reduced computation, the same construction can
be restricted to a user-specified set of spatial positions
\(\mathcal{R}\subseteq\mathcal{G}\):
\[
\Phi(g,c)=\mathrm{Sh}_c(w_g),
\qquad
g\in\mathcal{R}.
\]
The conditional game \(w_g\) remains unchanged because all positions outside
\(g\) stay fixed as the observed context. Restricting \(\mathcal{R}\) changes only the
evaluated support, not the definition of the position-wise attribution.

\section{Additional Experimental Details}
\label{app:exp_details}

\subsection{Datasets, Models, and Preprocessing}
\label{app:datasets-models}

\paragraph{CUB-200-2011.}
CUB-200-2011~\citep{wah2011caltech} contains 11,788 RGB images of 200 bird species. Each image is annotated with a bounding box, part keypoints, and binary attributes, including color attributes. We fine-tune an ImageNet-pretrained ResNet-18 classifier and use the ground-truth class logit as the scalar response for attribution. 

\paragraph{EuroSAT.}
EuroSAT~\citep{helber2019eurosat} contains 27,000 Sentinel-2 patches from 10 land-use/land-cover classes. We use the 13-band version, including visible RGB, red-edge, near-infrared, short-wave infrared, and atmospheric bands. We train a ResNet-18 classifier from scratch on all 13 bands. 

\paragraph{HLS Burn Scar.}
HLS Burn Scar~\citep{phillips2023hlsburnscar} contains Harmonized Landsat--Sentinel-2 patches with six spectral bands: Blue, Green, Red, NIR, SWIR\,1, and SWIR\,2. We train a lightweight UNet for burn-scar segmentation. For attribution and evaluation, we average the burn-class logit over the evaluation region to obtain a scalar response. 

\paragraph{Training and robustness.}
Full optimizer settings, learning rates, epochs, and preprocessing choices are dataset-specific and fixed before attribution evaluation.

\subsection{Attribution Baseline Details}
\label{app:baseline-details}

\paragraph{Integrated Gradients.}
We compute Integrated Gradients (IG)~\citep{sundararajan2017axiomatic} along a straight-line path from the masking baseline to the input using 300 integration steps. IG produces dense attributions in \(\mathbb{R}^{C\times H\times W}\). To match the common player space, we sum pixel-level attributions within each grid cell and channel, yielding one score for every position-channel player $(g,c)$.

\paragraph{DeepSHAP.}
DeepSHAP~\citep{lundberg2017unified,shrikumar2017learning} combines DeepLIFT-style propagation rules with Shapley-value approximation. We compute DeepSHAP with 1{,}000 background samples. As with IG, dense attributions are aggregated within each grid cell and channel to obtain scores over the same $(g,c)$ player space.

\paragraph{KernelSHAP.}
For KernelSHAP~\citep{lundberg2017unified}, each position-channel pair $(g,c)$ is treated as an interpretable feature. We use the same spatial grid as SSI and sample $2|\mathcal{P}|+2048$ coalitions following the library default, where the first $2|\mathcal{P}|$ are deterministic singleton and complement pairs that receive the highest Shapley kernel weights, and the remaining $2048$ are drawn uniformly at random. KernelSHAP is therefore the flat-Shapley baseline most directly matched to SSI's player granularity, but it does not impose the position-level coalition structure used by SSI.

\paragraph{RISE rank-1 control.}
We include a RISE-based rank-1 control~\citep{petsiuk2018rise}. RISE estimates a spatial importance map $s_{\mathrm{rise}}(g)$ by generating $M{=}8000$ random binary masks, evaluating the model response under each masked input, and computing a weighted average of the masks. We combine this spatial score with a global channel score $q(c)$, defined as the target-logit drop caused by masking channel $c$ globally across all spatial positions:
\[
A_{\mathrm{rise}}(g,c)=s_{\mathrm{rise}}(g)\,q(c).
\]
This gives a strong spatial perturbation baseline, but its rank-1 structure forces channel relevance to be global rather than varying by position. Comparing SSI against this control isolates the benefit of preserving the full joint attribution matrix rather than factorizing it into spatial and spectral summaries.

\paragraph{Additional gradient-based baselines on HLS Burn Scar.}
We additionally compare against Expected Gradients~\citep{erion2021expected},
which averages gradient-based contributions across reference samples and
interpolation points; Guided IG~\citep{kapishnikov2021guided}, which
adapts the integration path; and FullGrad~\citep{srinivas2019fullgradient},
which incorporates input and internal bias-gradient information.
All reported scores are evaluated over the same HLS Burn Scar
position-channel players and using the metrics of
Appendix~\ref{app:metric-definitions}.
For dense input-channel attributions, pixel-level scores are summed
within each grid cell while retaining their channel indices.
Unlike these input-channel attribution methods, FullGrad also includes
internal bias-gradient terms; its conversion to position-channel
scores is implementation-specific rather than a direct spatial sum of
a native input-channel tensor.

\paragraph{SSI.}
SSI computes the full-context attribution $\Phi(g,c)$ from Eq.~\eqref{eq:full-context} over the full position-channel player set. No search-selected support is used for the reported main metrics. The resulting map is evaluated directly as a joint attribution matrix and can also be summarized spatially by \(\sum_c \Phi(g,c)\) or spectrally by \(\sum_g \Phi(g,c)\).

\subsection{Computational Cost}
\label{app:computational_cost}

We distinguish the \emph{permutation count} $K$, the number of channel
orderings sampled in each cell's conditional game $w_g$ as used in
Proposition~\ref{prop:ssi_sampling_concentration}, from the \emph{evaluation
count} $N_{\mathrm{eval}}$, the total number of model forward passes needed
to produce all attribution scores for one image. Reading
Eq.~\eqref{eq:ssi-mc} as a literal recipe, each sampled permutation $\pi$
requires evaluating $w_g$ on the $C{+}1$ cumulative subsets
$\emptyset,\{c_1\},\{c_1,c_2\},\ldots,S_g$. Since $w_g(\emptyset){=}0$
requires no model evaluation and the endpoint $v(B_g)$ is shared across
all $K$ permutations at the same cell, the per-cell forward-pass count
is bounded by $K(C{+}1)$ in the worst case and is closer to
$K(C{-}1){+}2$ when endpoint states are cached. We report the worst-case
upper bound throughout to keep the comparison conservative; cached
implementations only reduce SSI's cost further. For the small-$C$
regime where SSI enumerates all subsets, the analogous bound is $2^C$
per cell.

Table~\ref{tab:forward_passes} reports $N_{\mathrm{eval}}$ for every
method under this convention. SSI enumerates all channel subsets when
$2^C\leq 10$ (CUB-200 with $C{=}3$, giving $G\cdot 2^C$ evaluations) and
switches to permutation sampling with $K{=}10$ otherwise (EuroSAT with
$C{=}13$ and HLS with $C{=}6$, giving up to $G\cdot K(C{+}1)$
evaluations). MC-Owen with $M{=}200$ block-respecting permutations
costs up to $M|\mathcal{P}|$ evaluations along the same conservative
counting convention, since each permutation traverses all
$|\mathcal{P}|$ players. Under this convention SSI uses approximately
$6\%$ of the MC-Owen budget on HLS Burn Scar
($4{,}480$ vs.\ $76{,}800$ evaluations); the gap remains one to two
orders of magnitude across all three datasets. These are analytical model-evaluation counts, not measured wall-clock times. In particular, forward-pass counts are not directly comparable with backward-pass costs of gradient-based methods, and the $6\%$ claim applies only to the comparison between SSI and MC-Owen under the stated counting convention.

\begin{table}[h]
\centering
\caption{\textbf{Forward-pass evaluations per image.}
$|\mathcal{P}|{=}G{\times}C$ is the total number of position-channel
players. SSI evaluations are reported under a worst-case convention
that does not assume cached endpoint states; cached implementations
further reduce SSI's cost. MC-Owen is shown at $M{=}200$ block-respecting
permutations.}
\label{tab:forward_passes}
\small
\setlength{\tabcolsep}{5pt}
\begin{tabular}{@{}l l r r r@{}}
\toprule
Method & Formula & CUB-200 & EuroSAT & HLS \\
       &         & \footnotesize$(16{\times}16,\,C{=}3)$
       & \footnotesize$(8{\times}8,\,C{=}13)$
       & \footnotesize$(8{\times}8,\,C{=}6)$ \\
\midrule
IG         & $300$                          &     300 &     300 &     300 \\
DeepSHAP   & $1{,}000$                      &   1,000 &   1,000 &   1,000 \\
KernelSHAP & $2|\mathcal{P}|{+}2048$        &   3,584 &   3,712 &   2,816 \\
RISE       & $8000{+}C$                     &   8,003 &   8,013 &   8,006 \\
SSI        & $G\cdot\min(2^C,\,K(C{+}1))$   &   2,048 &   8,960 &   4,480 \\
MC-Owen    & $M|\mathcal{P}|$               & 153,600 & 166,400 &  76,800 \\
\bottomrule
\end{tabular}
\end{table}

\subsection{Scalability to Higher Channel Counts}
\label{app:scalability}

SSI's computational cost per image scales as $G\cdot K(C{+}1)$ model
forward passes (Appendix~\ref{app:computational_cost}), where $K$ is
the per-cell permutation count. For $C \geq 4$, $K$ is capped at $10$,
so the per-image cost grows only linearly in $C$ rather than with
$|\mathcal{P}|=GC$ as in flat or Owen-style estimators that traverse
all players per permutation. The main concern is not how the computational cost grows with $C$, but
whether $K{=}10$ permutations provide sufficient estimation precision as
$C$ increases.

Proposition~\ref{prop:ssi_sampling_concentration} bounds the uniform
estimation error by
\[
R\sqrt{\frac{\ln(2|\mathcal{G}|C/\delta)}{2K}}.
\]
The dependence on $C$ enters only through $\sqrt{\ln C}$. Increasing
$C$ from $13$ (EuroSAT) to $200$ (a typical hyperspectral sensor)
multiplies the bound by $\sqrt{\ln 200 / \ln 13} \approx 1.3$, a
modest increase. Moreover, the marginal-contribution range $R$ is
expected to shrink as $C$ grows, since each channel then accounts for
a smaller fraction of the input signal, partially or fully offsetting
the $\sqrt{\ln C}$ growth.

If tighter control is needed for very large $C$, $K$ can be increased
at $O(\ln C)$ rate to maintain the same uniform bound. For example,
moving from $C{=}13$ to $C{=}200$ while holding the bound constant
requires only $K \approx 10 \times \ln 200 / \ln 13 \approx 21$
permutations per position, doubling the budget but remaining far below
the cost of flat or Owen-style estimators that scale with $GC$.

\section{Additional Evaluation Protocols}
\label{app:exp_metrics}

\subsection{Metric Definitions}
\label{app:metric-definitions}

\paragraph{Joint deletion}
Joint deletion starts from the original input and masks position-channel
players $(g,c)$ in decreasing order of signed attribution score. The target-class logit is recorded after each step. Lower deletion AUC and higher insertion AUC indicate better results.

\paragraph{Rank-1 projection.}
For each method, we reshape the grid-level attribution scores into a matrix $A\in\mathbb{R}^{C\times N}$, where $C$ is the number of channels and $N$ is the number of grid cells. We compute the best rank-1 approximation
\[
\widehat{A}=\sigma_1u_1v_1^\top,
\]
where $\sigma_1$ is the largest singular value and $u_1,v_1$ are the corresponding singular vectors. Deletion and insertion are then evaluated using $\widehat{A}$.

\paragraph{Spatial aggregation.}
Spatial maps are obtained by summing over channels:
\[
A_{\mathrm{sp}}(g)=\sum_c A(g,c).
\]
We threshold $A_{\mathrm{sp}}$ at multiple top-coverage levels and compute mIoU with available spatial annotations. CUB uses bird bounding boxes, and HLS uses burn-scar masks or their spatial support.

\paragraph{Spectral aggregation.}
Channel scores are obtained by summing over spatial cells:
\[
A_{\mathrm{ch}}(c)=\sum_g A(g,c).
\]
Channel deletion masks entire channels in decreasing $A_{\mathrm{ch}}(c)$ order. Spectral hit rates compare the top-ranked channels against dataset-specific spectral priors:
\[
\mathrm{Hit@}K
=
\frac{|\mathrm{TopK}(A)\cap \mathcal{T}_y|}{|\mathcal{T}_y|},
\]
averaged over correctly classified samples.

\subsection{ROAD Faithfulness Evaluation}
\label{app:road}

We also evaluate SSI with Remove and Debias (ROAD)~\citep{rong2022road}
on HLS Burn Scar as a complementary faithfulness check to the
standard masking-based deletion and insertion tests.
The evaluation compares methods on the same position-channel player
space, reporting ROAD area under the curve for removing the most
relevant features first (MoRF; lower is better) and the least relevant
features first (LeRF; higher is better).
ROAD is intended to mitigate some evaluation artifacts introduced by
feature removal, rather than eliminate every distribution-shift effect.
Table~\ref{tab:road_hls} reports the results included in the author
rebuttal. SSI performs best in both removal directions among the
nine methods considered.

\begin{table}[h!]
\centering
\caption{\textbf{ROAD evaluation on HLS Burn Scar.}
LeRF AUC (higher is better) and MoRF AUC (lower is better).
Values are mean $\pm$ reported variability, as in the author rebuttal.
Best scores per column are in \textbf{bold}.}
\label{tab:road_hls}
\small
\setlength{\tabcolsep}{10pt}
\begin{tabular}{@{}l cc@{}}
\toprule
Method & ROAD LeRF $\uparrow$ & ROAD MoRF $\downarrow$ \\
\midrule
SSI & \textbf{.657 $\pm$ .010} & \textbf{.1143 $\pm$ .0058} \\
MC-Owen & .629 $\pm$ .009 & .1275 $\pm$ .0058 \\
RISE & .624 $\pm$ .010 & .1293 $\pm$ .0067 \\
IG & .623 $\pm$ .010 & .1316 $\pm$ .0058 \\
DeepSHAP & .622 $\pm$ .009 & .1325 $\pm$ .0057 \\
Expected Gradients & .620 $\pm$ .010 & .1354 $\pm$ .0060 \\
Guided IG & .612 $\pm$ .012 & .1430 $\pm$ .0074 \\
KernelSHAP & .588 $\pm$ .010 & .1702 $\pm$ .0076 \\
FullGrad & .527 $\pm$ .015 & .2337 $\pm$ .0138 \\
\bottomrule
\end{tabular}
\end{table}

\subsection{Spectral Band Priors}
\label{app:spectral_priors}

The spectral hit-rate evaluation (\S\ref{sec:exp_aggregation}) requires a set of \emph{prior bands} for each class. These are channels whose reflectance is expected to be positively associated with the target-class decision. We select these priors from established spectral indices and physical reasoning rather than the attribution results themselves.

\paragraph{HLS Burn Scar.}
The HLS Burn Scar dataset uses a six-band stack (Blue, Green, Red, Narrow NIR, SWIR\textsubscript{1}, SWIR\textsubscript{2}) derived from the Harmonized Landsat Sentinel-2 product~\citep{claverie2018harmonized}. For the burn-scar class, we designate Red (band~2), SWIR\textsubscript{1} (band~4), and SWIR\textsubscript{2} (band~5) as prior bands. This selection follows directly from the Normalized Burn Ratio (NBR) and its variants, which are the standard spectral indices for burn severity mapping~\citep{key2006landscape,garcia1991mapping}. Burned surfaces exhibit elevated SWIR reflectance due to exposed soil, charcoal, and ash, and elevated red reflectance due to the loss of chlorophyll absorption~\citep{trigg2001mirbi}. These spectral signatures have been validated across decades of fire-mapping literature using Landsat TM/ETM+/OLI and Sentinel-2 MSI sensors. NIR is excluded from the positive prior because burn scars cause a sharp drop in NIR reflectance (loss of mesophyll scattering), making NIR \emph{negatively} correlated with the burn-scar class.

\paragraph{EuroSAT.}
EuroSAT provides 13-band Sentinel-2 imagery spanning the visible, red-edge, NIR, and SWIR spectral regions. We assign class-specific priors based on the physical reflectance properties of each land-cover type:
\begin{itemize}[nosep,leftmargin=*]
\item \textbf{Vegetation classes} (AnnualCrop, Forest, HerbaceousVegetation, Pasture, PermanentCrop). We use the red-edge (B05--B07, 705--783\,nm) and NIR (B08/B8A, 842/865\,nm) bands as priors. Healthy vegetation produces strong NIR reflectance through mesophyll scattering and exhibits a steep red-edge transition driven by chlorophyll absorption~\citep{clevers2002derivation,delegido2011evaluation}. These bands underpin NDVI, NDRE, and related vegetation indices that have been used operationally for crop monitoring and forest mapping for decades~\citep{tucker1979red,rouse1974monitoring}.
\item \textbf{Industrial.} The priors are the SWIR bands (B11, B12; 1610/2190\,nm). Impervious surfaces such as concrete, asphalt, and metal roofing reflect strongly in the SWIR due to low moisture content and high mineral reflectance~\citep{zha2003use,xu2008new}.
\item \textbf{Highway, Residential, River, SeaLake.} No positive prior is assigned. These classes are spectrally heterogeneous (Highway, Residential) or uniformly low-reflectance (water bodies), so no single band subset has a reliably positive association with the class decision.
\end{itemize}

In both datasets, we conservatively include only prior bands with a well-documented, physically grounded positive correlation with the target class. Classes without a clear positive spectral signature receive an empty prior set and are excluded from the hit-rate evaluation. This design ensures that the spectral validity metric tests whether the attribution method recovers known physical relationships, rather than fitting to arbitrary band rankings.

\subsection{Grid Resolution}
\label{app:grid_resolution}

Tables~\ref{tab:gr_hls_4x4}--\ref{tab:gr_euro_4x4} report faithfulness
metrics across alternative grid resolutions under the per-channel mean
baseline. For each dataset we vary the spatial grid from $4{\times}4$
to $8{\times}8$ (where applicable) and report joint deletion AUC,
rank-1 deletion AUC, and GridSum insertion AUC. SSI maintains leading
performance across resolutions, suggesting robustness to the choice of
grid size.

\begin{table}[!ht]
\centering
\caption{\textbf{Faithfulness test on HLS Burn Scar with $4{\times}4$ grid.}
Values are mean $\pm$ bootstrap std. Best in \textbf{bold}.}
\label{tab:gr_hls_4x4}
\small
\setlength{\tabcolsep}{5pt}
\begin{tabular}{@{}l ccc@{}}
\toprule
Method & Joint.Del$\downarrow$ & Rk1.Del$\downarrow$ & GridSum.Ins$\uparrow$ \\
\midrule
SSI        & \textbf{.1544$\pm$.0082} & \textbf{.1588$\pm$.0097} & \textbf{.6310$\pm$.0077} \\
IG         & .2499$\pm$.0159 & .2792$\pm$.0171 & .5739$\pm$.0089 \\
DeepSHAP   & .2481$\pm$.0145 & .2793$\pm$.0176 & .5835$\pm$.0089 \\
KernelSHAP & .2129$\pm$.0149 & .2534$\pm$.0179 & .5651$\pm$.0080 \\
RISE      & .4320$\pm$.0188 & .4330$\pm$.0190 & .5327$\pm$.0070 \\
\bottomrule
\end{tabular}
\end{table}

\begin{table}[!ht]
\centering
\caption{\textbf{Faithfulness test on CUB-200 with $4{\times}4$ grid.}
Values are mean $\pm$ bootstrap std. Best in \textbf{bold}.}
\label{tab:gr_cub_4x4}
\small
\setlength{\tabcolsep}{5pt}
\begin{tabular}{@{}l ccc@{}}
\toprule
Method & Joint.Del$\downarrow$ & Rk1.Del$\downarrow$ & GridSum.Ins$\uparrow$ \\
\midrule
SSI        & \textbf{.0239$\pm$.0011} & \textbf{.0264$\pm$.0013} & \textbf{.7267$\pm$.0111} \\
IG         & .0257$\pm$.0012 & .0284$\pm$.0015 & .6515$\pm$.0149 \\
DeepSHAP   & .0259$\pm$.0012 & .0279$\pm$.0013 & .6511$\pm$.0146 \\
KernelSHAP & .0263$\pm$.0012 & .0285$\pm$.0013 & .6900$\pm$.0124 \\
RISE      & .0329$\pm$.0017 & .0329$\pm$.0015 & .7109$\pm$.0113 \\
\bottomrule
\end{tabular}
\end{table}

\begin{table}[!ht]
\centering
\caption{\textbf{Faithfulness test on CUB-200 with $8{\times}8$ grid.}
Values are mean $\pm$ bootstrap std. Best in \textbf{bold}.}
\label{tab:gr_cub_8x8}
\small
\setlength{\tabcolsep}{5pt}
\begin{tabular}{@{}l ccc@{}}
\toprule
Method & Joint.Del$\downarrow$ & Rk1.Del$\downarrow$ & GridSum.Ins$\uparrow$ \\
\midrule
SSI        & \textbf{.0123$\pm$.0007} & \textbf{.0154$\pm$.0010} & \textbf{.7567$\pm$.0116} \\
IG         & .0152$\pm$.0008 & .0169$\pm$.0010 & .5215$\pm$.0213 \\
DeepSHAP   & .0151$\pm$.0008 & .0170$\pm$.0010 & .5204$\pm$.0221 \\
KernelSHAP & .0169$\pm$.0008 & .0240$\pm$.0016 & .5688$\pm$.0188 \\
RISE      & .0225$\pm$.0013 & .0225$\pm$.0012 & .7425$\pm$.0118 \\
\bottomrule
\end{tabular}
\end{table}

\begin{table}[!ht]
\centering
\caption{\textbf{Faithfulness test on EuroSAT with $4{\times}4$ grid.}
Values are mean $\pm$ bootstrap std. Best in \textbf{bold}.}
\label{tab:gr_euro_4x4}
\small
\setlength{\tabcolsep}{5pt}
\begin{tabular}{@{}l ccc@{}}
\toprule
Method & Joint.Del$\downarrow$ & Rk1.Del$\downarrow$ & GridSum.Ins$\uparrow$ \\
\midrule
SSI        & \textbf{.1295$\pm$.0158} & .1425$\pm$.0155 & \textbf{.7096$\pm$.0195} \\
IG         & .1957$\pm$.0233 & .1745$\pm$.0217 & .6731$\pm$.0224 \\
DeepSHAP   & .1963$\pm$.0234 & .1750$\pm$.0217 & .6730$\pm$.0223 \\
KernelSHAP & .1416$\pm$.0199 & \textbf{.1458$\pm$.0181} & .6432$\pm$.0240 \\
RISE      & .2116$\pm$.0272 & .2120$\pm$.0263 & .7007$\pm$.0188 \\
\bottomrule
\end{tabular}
\end{table}

\subsection{Full Spatial Localization Curves}
\label{app:full_spatial}

Tables~\ref{tab:cub_full_spatial}--\ref{tab:hls_full_spatial} report spatial mIoU at each coverage threshold. On CUB-200, SSI leads across all thresholds, while IG and DeepSHAP have average mIoU scores 3 to 4 points lower and KernelSHAP falls substantially behind. On HLS Burn Scar, SSI leads at all coverage levels by a clear margin, consistent with the localized nature of burn-scar evidence.

\begin{table}[!ht]
\centering
\caption{\textbf{CUB-200 spatial localization.} mIoU at top-coverage levels. Values are mean $\pm$ bootstrap std. Higher is better. Best scores are in \textbf{bold}.}
\label{tab:cub_full_spatial}
\small
\setlength{\tabcolsep}{4pt}
\begin{tabular}{@{}l c c c c c c@{}}
\toprule
Method & @0.1 & @0.2 & @0.3 & @0.4 & @0.5 & Avg. \\
\midrule
SSI        & \textbf{.2208$\pm$.0102} & \textbf{.3103$\pm$.0080} & \textbf{.3592$\pm$.0082} & \textbf{.3893$\pm$.0102} & \textbf{.4036$\pm$.0118} & \textbf{.3366$\pm$.0044} \\
IG         & .2148$\pm$.0083 & .2926$\pm$.0053 & .3228$\pm$.0054 & .3342$\pm$.0079 & .3361$\pm$.0099 & .3001$\pm$.0034 \\
DeepSHAP   & .2150$\pm$.0079 & .2929$\pm$.0050 & .3220$\pm$.0052 & .3333$\pm$.0074 & .3362$\pm$.0101 & .2999$\pm$.0033 \\
KernelSHAP & .1061$\pm$.0024 & .1787$\pm$.0030 & .2267$\pm$.0044 & .2674$\pm$.0071 & .2970$\pm$.0095 & .2152$\pm$.0026 \\
RISE       & .2099$\pm$.0090 & .2888$\pm$.0074 & .3293$\pm$.0072 & .3529$\pm$.0086 & .3671$\pm$.0103 & .3096$\pm$.0038 \\
\bottomrule
\end{tabular}
\end{table}

\begin{table}[!ht]
\centering
\caption{\textbf{HLS Burn Scar spatial localization.} mIoU at top-coverage levels. Values are mean $\pm$ bootstrap std. Higher is better. Best scores are in \textbf{bold}.}
\label{tab:hls_full_spatial}
\small
\setlength{\tabcolsep}{4pt}
\begin{tabular}{@{}l c c c c c c@{}}
\toprule
Method & @0.05 & @0.10 & @0.15 & @0.20 & @0.25 & Avg. \\
\midrule
SSI        & \textbf{.3109$\pm$.0113} & \textbf{.3886$\pm$.0131} & \textbf{.3874$\pm$.0137} & \textbf{.3587$\pm$.0159} & \textbf{.3321$\pm$.0155} & \textbf{.3555$\pm$.0065} \\
IG         & .2722$\pm$.0175 & .2699$\pm$.0162 & .2552$\pm$.0162 & .2396$\pm$.0168 & .2245$\pm$.0174 & .2523$\pm$.0075 \\
DeepSHAP   & .2850$\pm$.0177 & .2776$\pm$.0157 & .2583$\pm$.0157 & .2441$\pm$.0178 & .2270$\pm$.0174 & .2584$\pm$.0076 \\
KernelSHAP & .0809$\pm$.0102 & .0932$\pm$.0101 & .1075$\pm$.0115 & .1183$\pm$.0115 & .1228$\pm$.0127 & .1045$\pm$.0050 \\
RISE       & .1589$\pm$.0108 & .1913$\pm$.0116 & .2145$\pm$.0119 & .2346$\pm$.0113 & .2585$\pm$.0131 & .2116$\pm$.0048 \\
\bottomrule
\end{tabular}
\end{table}

EuroSAT lacks a spatially discriminative ground-truth proxy. All methods produce identical mIoU with zero bootstrap variance at every threshold, so we do not include a spatial localization table for this dataset.

\subsection{Full Spectral Hit Curves}
\label{app:full_spectral}

Tables~\ref{tab:hls_full_spectral} and \ref{tab:eurosat_full_spectral} report channel deletion AUC and spectral hit rates at all $K$ values. On HLS Burn Scar, SSI leads in channel deletion and achieves the highest hit rates, reaching .8867 at $K{=}5$ and exceeding the next-best method. On EuroSAT, SSI achieves the best channel deletion and leads on every hit-rate threshold, with KernelSHAP competitive at higher $K$ values.

\begin{table}[!ht]
\centering
\caption{\textbf{HLS Burn Scar spectral aggregation.} Lower channel deletion scores and higher hit rates are better. Values are mean $\pm$ bootstrap std. Best scores are in \textbf{bold}.}
\label{tab:hls_full_spectral}
\small
\setlength{\tabcolsep}{4pt}
\begin{tabular}{@{}l c c c c c c@{}}
\toprule
Method & Ch.Del. $\downarrow$ & Hit@1 & Hit@2 & Hit@3 & Hit@4 & Hit@5 \\
\midrule
SSI        & \textbf{.2180$\pm$.0094} & \textbf{.3167$\pm$.0068} & \textbf{.4333$\pm$.0158} & \textbf{.6033$\pm$.0159} & \textbf{.7567$\pm$.0193} & \textbf{.8867$\pm$.0163} \\
IG         & .2724$\pm$.0178 & .3000$\pm$.0092 & .3667$\pm$.0148 & .4200$\pm$.0154 & .5767$\pm$.0148 & .7600$\pm$.0148 \\
DeepSHAP   & .2723$\pm$.0174 & .3000$\pm$.0103 & .3667$\pm$.0145 & .4233$\pm$.0160 & .5767$\pm$.0147 & .7600$\pm$.0151 \\
KernelSHAP & .2568$\pm$.0161 & .2833$\pm$.0120 & .3567$\pm$.0126 & .4633$\pm$.0163 & .6067$\pm$.0167 & .7533$\pm$.0145 \\
RISE       & .3166$\pm$.0204 & .2333$\pm$.0153 & .3067$\pm$.0200 & .5000$\pm$.0192 & .5533$\pm$.0159 & .6700$\pm$.0034 \\
\bottomrule
\end{tabular}
\end{table}

\begin{table}[!ht]
\centering
\caption{\textbf{EuroSAT spectral aggregation.} Lower channel deletion scores and higher hit rates are better. Values are mean $\pm$ bootstrap std. Best scores are in \textbf{bold}.}
\label{tab:eurosat_full_spectral}
\small
\setlength{\tabcolsep}{4pt}
\begin{tabular}{@{}l c c c c c c@{}}
\toprule
Method & Ch.Del. $\downarrow$ & Hit@1 & Hit@2 & Hit@3 & Hit@4 & Hit@5 \\
\midrule
SSI        & \textbf{.1039$\pm$.0165} & \textbf{.1768$\pm$.0231} & \textbf{.3472$\pm$.0236} & \textbf{.4545$\pm$.0323} & \textbf{.5328$\pm$.0289} & \textbf{.5960$\pm$.0264} \\
IG         & .1426$\pm$.0231 & .0258$\pm$.0115 & .2220$\pm$.0283 & .3061$\pm$.0313 & .3881$\pm$.0322 & .4624$\pm$.0335 \\
DeepSHAP   & .1437$\pm$.0241 & .0258$\pm$.0113 & .2220$\pm$.0291 & .3061$\pm$.0306 & .3843$\pm$.0341 & .4573$\pm$.0358 \\
KernelSHAP & .1347$\pm$.0209 & .0879$\pm$.0205 & .2207$\pm$.0298 & .3295$\pm$.0330 & .4232$\pm$.0335 & .5394$\pm$.0366 \\
RISE       & .1387$\pm$.0207 & .0788$\pm$.0181 & .1876$\pm$.0241 & .2629$\pm$.0288 & .3311$\pm$.0270 & .4258$\pm$.0242 \\
\bottomrule
\end{tabular}
\end{table}

\end{document}